\documentclass{article}

\PassOptionsToPackage{numbers}{natbib}

\usepackage[preprint]{neurips_2026}
\usepackage{bm}

\usepackage[utf8]{inputenc} 
\usepackage[T1]{fontenc}    
\usepackage[hypertexnames=false]{hyperref}       
\usepackage{url}            
\usepackage{booktabs}       
\usepackage{amsfonts}       
\usepackage{nicefrac}       
\usepackage{microtype}      
\usepackage{xcolor}         

\usepackage{amsmath, amssymb}
\usepackage{enumitem}
\usepackage{amsmath}
\usepackage{amsthm}
\usepackage{bm}
\usepackage{thm-restate}
\usepackage{thmtools}
\usepackage{graphicx}

\newtheorem{theorem}{Theorem}[section]
\newtheorem{proposition}[theorem]{Proposition}
\newtheorem{lemma}{Lemma}[section]
\newtheorem{corollary}{Corollary}[section]
\theoremstyle{definition}
\newtheorem{definition}{Definition}[section]
\newtheorem{assumption}{Assumption}[section]
\theoremstyle{remark}
\newtheorem{remark}{Remark}[section]

\newcommand{\Real}{\mathbb{R}}
\newcommand{\vect}[1]{\bm{#1}}
\newcommand{\vtheta}{\vect{\theta}}
\newcommand{\f}[2]{f\left(#1;#2\right)}
\newcommand{\g}[1]{g\left(#1\right)}
\newcommand{\ginv}[1]{g|^{-1}_{\vtheta_0}\left(#1\right)}
\newcommand{\J}[1]{J\left(#1\right)}
\newcommand{\K}[1]{K\left(#1\right)}
\newcommand{\gf}[1]{\phi_{\mathrm{gf}}\left(#1,t\right)}
\newcommand{\orbit}[1]{\mathcal{O}_{\mathrm{flow}}\left(#1\right)}
\newcommand{\M}[1]{\mathcal{M}_\mathrm{flow}\left(#1\right)}
\newcommand{\fibre}[1]{\Theta\left(#1\right)}
\newcommand{\Loss}[2]{\mathcal{L}\left(#1,#2\right)}
\newcommand{\norm}[1]{\left\lVert#1\right\rVert}
\newcommand{\dist}[4]{d^{#1}_{#2}\left(#3, #4\right)}
\newcommand{\vu}{\vect{u}}

\title{Canonical Regularisation of Wide Feature-Learning Neural Networks}

\author{%
  George Whittle$^\dagger$, Pranav Vaidhyanathan, Juliusz Ziomek, Natalia Ares, Maike A. Osborne \\
  Department of Engineering Science\\
  University of Oxford\\
  $^\dagger$ Corresponding Author: \texttt{george.whittle@reuben.ox.ac.uk} \\
}

\begin{document}

\maketitle

\begin{abstract}
Wide neural networks in the feature-learning regime drive modern deep learning, and yet they remain far less studied than their kernel-regime counterparts. We consider a critical yet under-explored difference between these two regimes: the regulariser and prior implied by gradient flow training. This \emph{canonical regularisation} property is well-studied in kernel regime networks --- of all the infinite global minima, gradient flow selects exactly the vanishing ridge solution --- and underpins the celebrated NN-GP correspondence, precisely allowing the modelling of noise during training. However, we prove ridge regularisation biases gradient flow in feature-learning regime networks, even in the infinitesimal limit of vanishing regularisation. Over training, ridge distorts the inductive bias of the network, with a particular damage done to pretrained networks where the implicit prior is informative. We resolve this by axiomatising the canonical regulariser as a regime-agnostic function-space energy and lift, which uniquely identifies ridge in the kernel regime, and crucially generalises to the feature-learning regime. By studying the Riemannian geometry of feature-learning networks, we derive \emph{geodesic ridge} from our framework, generalising ridge to the feature-learning regime. Correspondingly, we prove the canonical function-space prior is a \emph{Riemannian Gibbs Process}, generalising the more familiar Gaussian Process. As a practical contribution, we propose \emph{arc ridge} as a minimax-robust, scalable surrogate to geodesic ridge, revealing a deep relationship between early stopping and canonical regularisation across learning regimes. Finally, we demonstrate the consequences of our theory empirically on both image processing and NLP transfer-learning problems.
\end{abstract}

\section{Introduction}
\label{sec:introduction}

\begin{figure}
  \centering
  \begin{minipage}{0.49\textwidth}
    \centering
    \includegraphics[width=\linewidth,trim={0 1.5cm 0 3cm},clip]{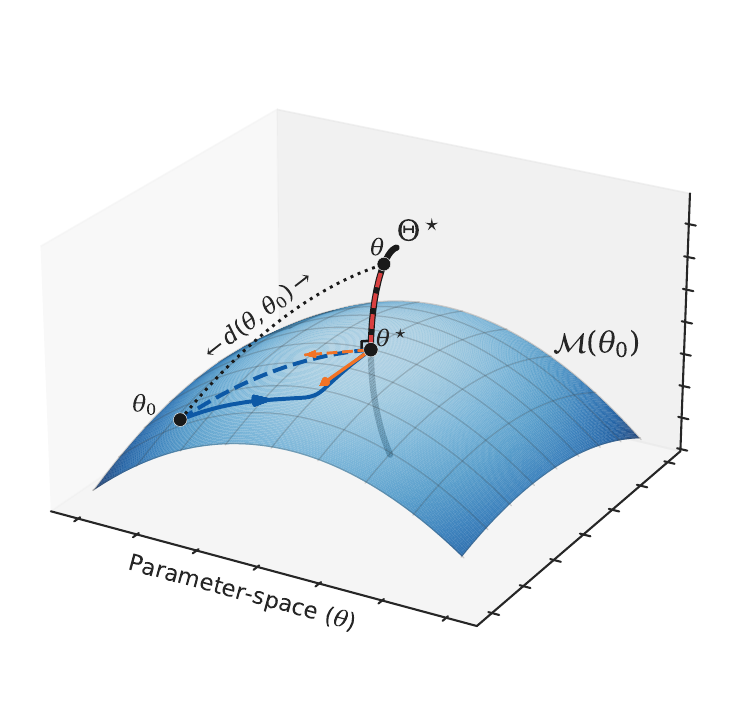}
\end{minipage}
\hfill
\begin{minipage}{0.49\textwidth}
    \centering
    \includegraphics[width=\linewidth,trim={0 1.5cm 0 3cm},clip]{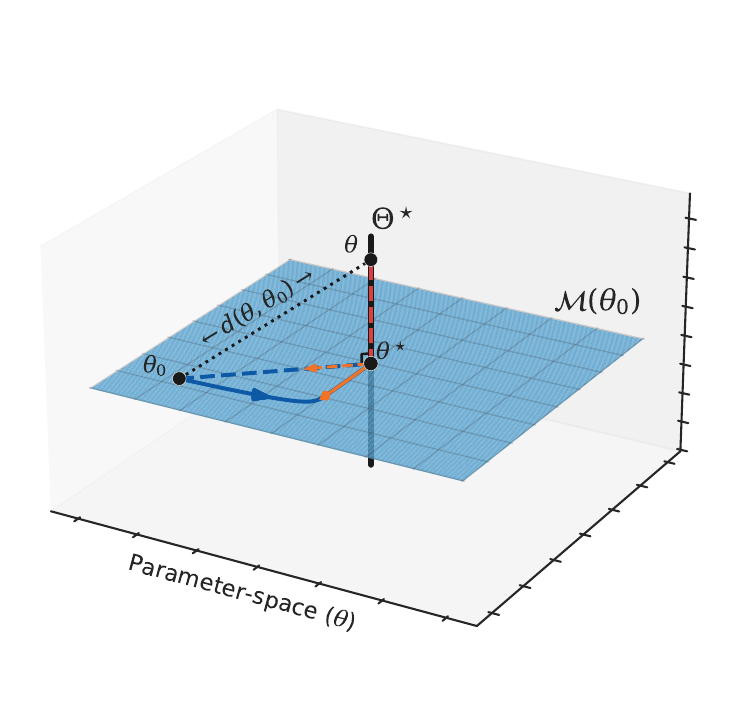}
\end{minipage}
  \caption{\emph{Left}: feature-learning regime geometry. The flow manifold $\M{\vtheta_0}$ (blue surface) is a curved $n$-dimensional submanifold near $\vtheta^\star$; $\Theta^\star$ (solid black curve) is the interpolating set, that is, global minima. The gradient flow trajectory (solid blue curve) lies on $\M{\vtheta_0}$ and converges to $\vtheta^\star$. The geodesic distance $\dist{}{\mathrm{flow}}{\vtheta_0}{\vtheta^\star}$ on $\M{\vtheta_0}$ (dashed blue curve) underpins the canonical regulariser. The regularisation gradients at $\vtheta^\star$, $\nabla_{\vtheta} \dist{2}{}{\vtheta}{\vtheta_0}$ (dashed) and $\nabla_{\vtheta} L^2(\gamma_\text{train})$ (solid), are shown in orange; both lie in $T_{\vtheta^\star}\M{\vtheta_0}$ along the realised gradient flow direction, and so do not bias gradient flow as standard and anchored ridge do. \emph{Right}: kernel regime geometry. The flow manifold is an affine subspace---geodesics are straight lines, and the canonical regulariser reduces to anchored ridge.}
  \label{fig:geometry}

  \vspace{-0.5cm}
\end{figure}

Wide neural networks now drive nearly every advance in modern deep learning, yet our theoretical understanding of these networks operating in the regime in which they actually \emph{learn features}---rather than evolving as effectively fixed kernel machines---remains strikingly limited. The kernel regime, formalised by the Neural Tangent Kernel (NTK) of \citet{jacot2018neural} and elaborated by \citet{lee2019wide} and \citet{du2019gradient}, supports a complete account of training dynamics, generalisation, and Bayesian inference; the feature-learning regime supports no comparable toolkit. This asymmetry would be of mainly theoretical interest were it not that many practices inherited from the kernel-regime continue to be applied in the feature-learning regime by default, despite the foundations on which they rest no longer holding. This paper confronts one such case---\emph{which regulariser, and which function-space prior, gradient-flow training implicitly favours}---and identifies the geometric reason.

The kernel-regime picture is well understood. An overparametrised model has an entire continuum of zero-loss interpolators, yet gradient flow selects a single, well-defined point from this continuum: the unique vanishing-regularisation limit of \emph{anchored ridge}, $\norm{\vtheta-\vtheta_0}^2$, the squared parameter-space distance from initialisation~\citep{osband2018randomized,jacot2018neural,he2020bayesian,ordonez2026gaussian}. This identification runs deeper than coincidence: vanishing-regularisation alone is not sharp enough to pin out anchored ridge from its uncountably-many quadratic competitors (Proposition~\ref{prop:naturality-nogo}); a pair of natural axioms on the function-space energy (Theorem~\ref{thm:output-energy-uniqueness}) is needed to single it out. Anchored ridge equals the squared Reproducing Kernel Hilbert Space (RKHS) norm of the change in network output under the NTK over training, so the same penalty governs both parameter and function space, and the corresponding implicit prior over functions is Gaussian---the celebrated Neural Network--Gaussian Process correspondence. From this single equivalence flow virtually all the analytical conveniences of kernel-regime networks: principled treatment of label noise, optimal shrinkage, calibrated uncertainty, predictable transfer behaviour. Anchored ridge is, in this sense, the \emph{canonical} regulariser of the kernel regime. Standard ridge $\norm{\vtheta}^2$~\citep{hoerl1970ridge,tikhonov1963solution} recovers the same selection only when $\vtheta_0=\vect{0}$; for any non-zero initialisation it already biases the gradient-flow limit \emph{within} the kernel regime (Remark~\ref{rmk:weight-decay-fails}, Figure~\ref{fig:trajectories}, right).

Outside the kernel regime, the same picture decomposes. \citet{yang2021tensor} establish a sharp dynamical dichotomy: in the regime in which networks actually learn features and transfer between tasks, the Jacobian evolves throughout training, the kernel depends on optimisation history, parameter trajectories leave any single affine plane, and the function-space identity underlying anchored ridge's central role breaks. In this regime, formalised by the Tensor Programs framework~\citep{yang2019wide,yang2020tensor,yang2020tensor3,yang2021tensor} and complemented by mean-field analyses~\citep{chizat2018global,mei2018mean}, no analogous canonical regulariser has been identified. Practitioners apply weight decay~\citep{hoerl1970ridge} or anchored ridge by default, even though the kernel-regime arguments justifying these choices no longer apply. The consequences are not abstract: weight decay is implicated in the grokking phenomenon~\citep{power2022grokking,xu2026grok} and in a documented degradation of pretrained representations during fine-tuning~\citep{kumar2022fine,lauditi2025adaptive}.

We argue that the gap is geometric. Behind every overparametrised network lie three natural parameter-space objects (Section~\ref{subsec:geometry}): the \emph{flow orbit}, the region of parameter space gradient flow can ever explore from a given initialisation; the \emph{gauge-fixed flow manifold}, the canonical $n$-dimensional surface within this orbit that contains the realised trajectory and stands in one-to-one correspondence with output values near the interpolator---in effect, a parameter-space chart of function space; and the \emph{output fibres}, directions transverse to the flow manifold along which predictions on training points do not change, but along which predictions at unseen inputs can change drastically. In the kernel regime these collapse: orbit, flow manifold, and the affine plane $\vtheta_0+\mathrm{Im}(J^\top)$ coincide, the trajectory is confined to it, and the fibres meet it orthogonally everywhere. Anchored ridge, the squared chord from $\vtheta_0$ to $\vtheta$, lives entirely within this plane (Figure~\ref{fig:geometry}, right); its gradient never deflects parameters along the fibres, so the gradient-flow limit is preserved as the regularisation strength shrinks. In the feature-learning regime the flow orbit can grow beyond $n$ dimensions, gauge-fixing extracts a curved $n$-dimensional manifold (Figure~\ref{fig:geometry}, left), and the trajectory bends through this surface. Anchored ridge, still measured along Euclidean chords, no longer respects this curvature: at the gradient-flow limit its gradient acquires a component along the output fibres, pulling parameters along directions invisible to the training loss but consequential at test points. The resulting bias is silent at the loss level, persists in the limit as regularisation shrinks to zero, and is what we prove formally (Theorem~\ref{thm:ridge-biases}).

The same geometry suggests the cure. The intrinsic counterpart of squared Euclidean distance on a curved manifold is squared geodesic distance, and the corresponding intrinsic counterpart of anchored ridge is \emph{geodesic ridge} (Proposition~\ref{prop:geodesic-ridge}): squared geodesic distance from initialisation along the flow manifold, plus an analogous on-fibre term. Its gradient at the gradient-flow limit lies tangent to the flow manifold along the realised flow direction (Figure~\ref{fig:geometry}, left, dashed orange), and through the standard MAP correspondence it yields a function-space prior that generalises the NTK Gaussian Process to curved geometry: the \emph{Riemannian Gibbs Process}. In the kernel regime the manifold flattens, geodesics straighten into chords, and both objects reduce to anchored ridge and the NTK GP. The progression standard ridge $\to$ anchored ridge $\to$ geodesic ridge thus captures a sequence of geometric corrections: one Euclidean for initialisation, then one Riemannian for manifold curvature.

Geodesic distances are computationally inaccessible at scale. Our final contribution identifies a tractable upper bound that retains the alignment property: \emph{arc ridge}, the squared total path length of the training trajectory itself. Arc ridge requires no manifold solvers and no second-order information; it is a single scalar accumulated over training updates. It nevertheless inherits the no-bias property of geodesic ridge along the realised gradient-flow direction, sandwiches geodesic ridge from above against an anchored-ridge lower bound, sits at the minimax-optimal point within this sandwich, and reduces formally to early stopping under exact gradient flow. The combination places it within reach of any standard training pipeline.

\paragraph{Contributions.}
\begin{itemize}[nosep]
    \item \textbf{A geometric framework for wide feature-learning networks.} We characterise the parameter-space geometry of any overparametrised network satisfying a single uniform-conditioning assumption (Assumption~\ref{ass:sigma-min}) via the orbit / gauge-fixed flow manifold / output fibre triple of Section~\ref{sec:problem-setting}. The kernel regime is recovered as the affine special case in which all three coincide with $\vtheta_0+\mathrm{Im}(J^\top)$.
    \item \textbf{Both standard and anchored ridge provably bias gradient flow.} Standard ridge fails in both regimes (Remark~\ref{rmk:weight-decay-fails}, Figure~\ref{fig:trajectories}, right): in the kernel regime by ignoring initialisation, in the feature-learning regime additionally through curvature. Anchored ridge---the canonical kernel-regime regulariser---fails in the feature-learning regime because manifold curvature deflects its gradient along output fibres invisible to the training loss, biasing the gradient-flow limit by an amount that does not vanish as regularisation shrinks (Theorem~\ref{thm:ridge-biases}, Figure~\ref{fig:trajectories}, left). This challenges the kernel-regime foundations on which much common regularisation practice rests.
    \item \textbf{A first-principles characterisation that uniquely identifies anchored ridge in the kernel regime and uniquely generalises it.} We axiomatise the canonical function-space energy as the minimum control-theoretic energy required to transport network outputs from initialisation to a target (Definition~\ref{def:canonical-energy}), and prove uniqueness in the kernel regime under invariance and orthogonal-additivity axioms (Theorem~\ref{thm:output-energy-uniqueness}), recovering anchored ridge and the NTK Gaussian Process exactly. The same construction lifts to parameter space (Definition~\ref{def:canonical-lift}) and extends to the feature-learning regime, identifying \emph{geodesic ridge} (Proposition~\ref{prop:geodesic-ridge}) as the canonical regulariser and the \emph{Riemannian Gibbs Process} as the canonical function-space prior, with the unbiased gradient-flow selection property preserved (Theorem~\ref{thm:canonical-characterisation}).
    \item \textbf{Arc ridge: a scalable, minimax-robust surrogate.} Squared training-path length upper-bounds geodesic ridge along $\M{\vtheta_0}$ (Theorem~\ref{thm:path-length-upper}), retains the no-bias property along the realised gradient-flow direction, sits at the minimax-optimal point of the sandwich between anchored ridge below and arc ridge above (Proposition~\ref{prop:minimax-robust}), and is formally equivalent to early stopping under exact gradient flow. Empirically (Section~\ref{sec:experiments}), arc ridge recovers the predicted advantages on transfer-learning settings where the implicit prior carries information.
\end{itemize}

\section{Problem Setting}
\label{sec:problem-setting}

We study overparametrised supervised models $\f{\vect{x}}{\vtheta} : \mathcal{X} \times \Theta \to \mathcal{Y}$ with scalar output $\mathcal{Y}\subseteq\Real$, trained on some dataset $\mathcal{D} = (X, Y) \in (\mathcal{X} \times \mathcal{Y})^n$ with $n \ll m=\dim\Theta$ finite and distinct inputs $X_i \ne X_j$ for $i \ne j$. The input space $\mathcal{X}$ is arbitrary. We assume $f$ is expressible in the \emph{Tensor Programs Framework}~\citep{yang2019wide,yang2020tensor,yang2020tensor3,yang2021tensor}, and possesses only smooth non-linearities. We write $\g{\vtheta} : \Theta \to \mathcal{Y}^n$ for the stacked output vector with $\g{\vtheta}_i = \f{X_i}{\vtheta}$. The model-output Jacobian is $\J{\vtheta} \in \Real^{n \times m}$ with $\J{\vtheta}_{ij} = \frac{\partial \g{\vtheta}_i}{\partial \vtheta_j}$, and the instantaneous NTK Gram matrix is $\K{\vtheta} = \J{\vtheta}\J{\vtheta}^\top \in \Real^{n \times n}$.

We study supervised training under gradient flow dynamics with the canonical SSE loss $\Loss{\g{\vtheta}}{Y}=\norm{\g{\vtheta}-Y}^2$, i.e. $-\J{\vtheta}^\top\nabla_{\g{\vtheta}}\Loss{\g{\vtheta}}{Y}=-2\J{\vtheta}^\top\left(\g{\vtheta}-Y\right)$, corresponding to function-space dynamics $\dot g(\vtheta) = -2\K{\vtheta}\left(\g{\vtheta}-Y\right)$. We denote the value of $\vtheta$ after running gradient flow from $\vtheta_0$ for a time $t$ as $\gf{\vtheta_0}$. We assume $\mathcal{L}$ is minimised by an \emph{interpolating} solution, achieved by any $\vtheta^\star\in\Theta^\star=\{\vtheta : \g{\vtheta} = Y\}$. Note further that extensions of our results to strictly convex, smooth losses are deferred to the appendix.

Given the above setup, our results rest on only a single assumption:

\begin{assumption}[Uniform Jacobian conditioning]
\label{ass:sigma-min}
There exist an open neighbourhood $U \supset \Theta^\star$ and constants $0 < c \le C < \infty$ such that for all $\vtheta \in U$,
\[
    c \;\le\; \sigma_{\min}(\J{\vtheta}) \quad\text{and}\quad \sigma_{\max}(\J{\vtheta}) \;\le\; C.
\]
Furthermore, assume that for any given initialisation $\vtheta_0$, $\gf{\vtheta_0}\in U\;\forall\;t\in[0,\infty)$.
\end{assumption}

This is a standard assumption in wide-network analysis: all infinite-width kernel-regime literature assumes it, and it is justified in the feature-learning regime as under $\mu P$ scaling for Tensor Programs networks with smooth activations~\citep{yang2020tensor,yang2020tensor3}, $\sigma_{\min}(\J{\vtheta_0})$ concentrates around its positive infinite-width limit at initialisation, and continuity propagates the bound to a neighbourhood. Furthermore, we verify this empirically along flow trajectories in Appendix~\ref{app:assumption-verification}. \footnote{An immediate useful consequence of Assumption~\ref{ass:sigma-min} is that gradient flow converges to an interpolating solution: $\vtheta^\star:=\lim \limits_{t\to\infty} \gf{\vtheta_0}\in\Theta^\star$. Indeed, the loss is a valid Lyapunov function with time-derivative $\dot{\mathcal{L}}=-2\left(\g{\vtheta}-Y\right)^\top\K{\vtheta}\left(\g{\vtheta}-Y\right)$. Since $\J{\vtheta}$ has full rank along the trajectory, $\K{\vtheta}:=\J{\vtheta}\J{\vtheta}^\top\succeq c^2 I\succ0$, so $\dot{L}< 0$ whenever $\g{\vtheta}\neq Y$. Convergence to some $\vtheta^\star\in\Theta^\star$ thus follows, in fact exponentially with rate at least $2c^2$.}

\section{Ridge Biases Gradient Flow in the Feature Learning Regime}
\label{sec:ridge-pathology}

The first core theoretical contribution of this work is an exposition of the feature-learning regime as a \emph{pathology} of both anchored and standard ridge. To understand this, we must first understand the \emph{geometry} underlying wide networks, and how it differs between the two regimes.

\subsection{Parameter-space Geometry}
\label{subsec:geometry}
We begin by illuminating the parameter-space geometry of any overparametrised network satisfying Assumption~\ref{ass:sigma-min}, restricting ourselves to the region $U$ for complete rigour. This geometry is best understood as a collection of lower-dimensional subsets, which collectively span all of $U$. The first of these is the horizontal orbit generated by the directions available to gradient flow:

\begin{definition}[Flow orbit]
Let
\[
    \mathcal H_{\vtheta}:=\operatorname{Im} J(\vtheta)^\top \subset T_{\vtheta}\Theta
\]
denote the horizontal distribution induced by the model-output Jacobian. The \emph{flow orbit}
through $\vtheta_0$, denoted $\orbit{\vtheta_0}\subseteq\Theta$, is the
reachable set of the horizontal control system
\[
    \dot{\vtheta}(t)=J(\vtheta(t))^\top \vu(t),
    \qquad \vu\in L^2([0,T],\mathbb R^n),
\]
namely
\[
    \orbit{\vtheta_0}
    :=
    \left\{
    \vtheta(T):
    T>0,\ 
    \vu\in L^2([0,T],\mathbb R^n),\
    \dot{\vtheta}(t)=J(\vtheta(t))^\top\vu(t),\
    \vtheta(0)=\vtheta_0
    \right\}.
\]
\end{definition}

The flow orbit is the fundamental control-theoretic object. In general it is a sub-Riemannian reachable set as its tangent directions are generated by the Lie closure of $\mathcal H$, and therefore its dimension can exceed $n$. Thus the orbit itself should not be identified with function space. In order to move from function-space quantities back to parameter space, we fix a gauge, selecting a single representative from the flow orbit for each output value. This gives us our second fundamental geometric object:

\begin{definition}[Gauge-fixed flow manifold and representative map]
\label{def:flow-manifold}
Let
\[
    \vtheta^\star
    :=
    \lim_{t\to\infty}\gf{\vtheta_0}
\]
be the gradient-flow interpolator from $\vtheta_0$. The \emph{gauge-fixed flow manifold}
through $\vtheta_0$ is the basin sheet
\[
    \M{\vtheta_0}
    :=
    \operatorname{cc}_{\gamma_{\mathrm{gf}}}
    \left\{
    \vtheta\in \orbit{\vtheta_0}:
    \lim_{t\to\infty}\gf{\vtheta}
    =
    \lim_{t\to\infty}\gf{\vtheta_0}
    \right\} \subseteq  \orbit{\vtheta_0}\,
\]
where $\operatorname{cc}_{\gamma_{\mathrm{gf}}}$ denotes the connected component containing
the realized gradient-flow trajectory
\[
    \gamma_{\mathrm{gf}}:=\{\gf{\vtheta_0}:t\in[0,\infty)\}.
\]

Under Assumption~\ref{ass:sigma-min}, $\Theta^\star \cap U$ is normally hyperbolic, and thus there exists an open neighbourhood $V$ of $\Theta^\star \cap U$ in which $\M{\vtheta_0}$ is an embedded submanifold by the stable manifold theorem. For this neighbourhood, we define
\[
    C_0(\vtheta_0)
    :=
    g\bigl(\M{\vtheta_0}\cap V\bigr)
    \subseteq \mathcal Y^n .
\]
The canonical representative map is
\[
    \ginv{\vect c}
    :=
    \left(
    g\big|_{\M{\vtheta_0}\cap V}
    \right)^{-1}(\vect c),
    \qquad
    \vect c\in C_0(\vtheta_0).
\]
\end{definition}

\begin{figure}
  \centering
  \begin{minipage}{0.49\textwidth}
    \centering
    \includegraphics[width=\linewidth,trim={0 0.45cm 0 0.4cm},clip]{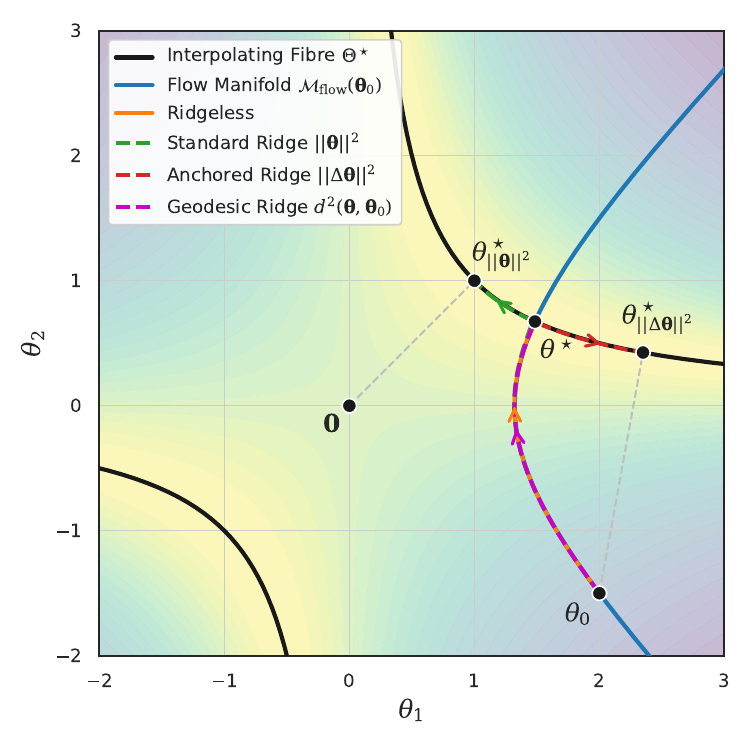}
\end{minipage}
\hfill
\begin{minipage}{0.49\textwidth}
    \centering
    \includegraphics[width=\linewidth,trim={0 0.45cm 0 0.4cm},clip]{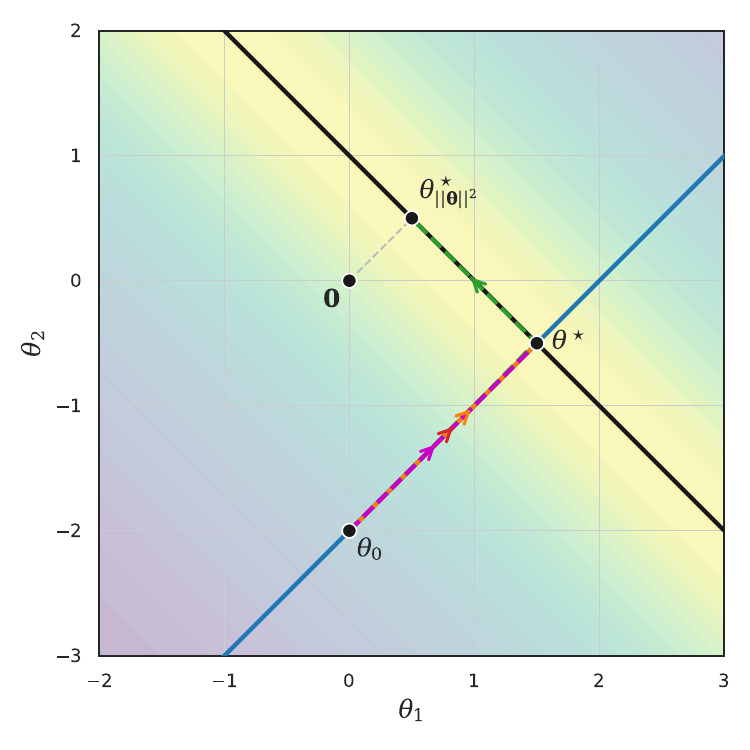}
\end{minipage}
  \caption{Gradient flow trajectories for the minimal overparametrised non-linear model $\f{x}{\vtheta}=\theta_1\theta_2x$ (left) and linear model $\f{x}{\vtheta}=\left(\theta_1+\theta_2\right)x$ (right) on the singleton dataset $\mathcal{D}=\{(1,1)\}$ for a variety of infinitesimal-strength regularisers, over a background shading indicating data loss. As our theory predicts, only the canonical regulariser prevents a biased equilibrium in the non-linear case, with both standard and anchored ridge causing flow along the interpolating fibre after convergence of the data loss. Additionally as expected, in the linear case anchored ridge coincides exactly with the canonical regulariser due to their equivalence, while standard ridge continues to bias the solution.}
  \label{fig:trajectories}

\vspace{-0.5cm}
\end{figure}

Intuitively, $\orbit{\vtheta_0}$ contains all parameters reachable by horizontal motion, while $\M{\vtheta_0}$ chooses the basin-compatible representative sheet. The map $\ginv{\vect c}$ selects the representative of output $\vect c$ lying on this sheet. Near $\vtheta^\star$, this representative is unique for every $\vect c\in C_0(\vtheta_0)$. A global notation $\ginv{\vect c}$ on all of $\mathcal Y^n$ is justified whenever this basin sheet is globally single-valued over output space; otherwise all statements are understood on the local output domain $C_0(\vtheta_0)$.

Our final geometric object spans the remainder of parameter-space:

\begin{definition}[Output fibre]
The \emph{output fibre} through $\vtheta$, denoted $\fibre{\vect{c}}\subset\Theta$ where $\vect{c}=\g{\vtheta}$, is the preimage of $\vect{c}$ under $g$:
\[
    \fibre{\vect{c}}:=g^{-1}\left(\vect{c}\right)=\bigl\{\,\vtheta:\g{\vtheta}=\vect{c}\bigr\}.
\]
\end{definition}

Intuitively, $\fibre{\vect{c}}$ represents the collection of vertical directions in parameter space which do not affect training predictions $\g{\vtheta}$, but can have a marked effect on the network's predictions away from the training inputs $X$. These directions will be of particular interest to us, as any parameter motion here is invisible to the training loss, but can dramatically affect the generalisation error. Geometrically, these output fibres are embedded submanifolds on $U$ of codimension $n$ (that is, they are smooth surfaces of dimension $m-n$) as $\vect{c}$ is defined as a regular value of $g$, satisfy $\fibre{Y}=\Theta^\star \cap U$, and have tangent space $\ker \J{\vtheta}$ so are everywhere orthogonal to horizontal gradient flow trajectories.

\subsection{The failure of ridge in the feature-learning regime}
Equipped with our geometric understanding of parameter-space, exposing the failure of both standard and anchored ridge in the feature learning regime is straightforward. The feature-learning regime is characterised by \emph{curved} geometry --- features evolve over training and $J$ changes meaningfully with $\vtheta$ --- thus $\M{\vtheta_0}$ curves meaningfully between $\vtheta_0$ and $\vtheta^\star$. As illustrated in Figure~\ref{fig:geometry}, this stands in stark contrast to the geometry of the kernel regime where $J$ is constant on $U$ and $\M{\vtheta_0}\equiv\orbit{\vtheta_0}=\left\{\vtheta_0 + J^\top \vu  : \vu  \in \Real^n\right\}$ is an affine subspace.

Now consider the interaction of these geometries with the anchored ridge regulariser, $\norm{\vtheta-\vtheta_0}^2$. The contribution to gradient flow training of this regulariser is $-2\left(\vtheta-\vtheta_0\right)$, e.g. the chord from $\vtheta_0$ to $\vtheta$. In the kernel regime, this always lies within the tangent space of $\M{\vtheta_0}$, and thus anchored ridge can never induce parameter movement off-manifold --- it is \emph{compatible} with gradient flow (Proposition~\ref{prop:linear-recovery}). This geometric alignment is precisely what causes the vanishing regularisation property underpinning the NN-GP correspondence, and as shown by \citet{ordonez2026gaussian}, allows us to explicitly model observation noise in the training set through an MAP interpretation of training.

In the feature-learning regime, however, curvature of $\M{\vtheta}$ causes this gradient to generically have a component lying in off-manifold fibre directions. The implications of this are significant: the regularisation gradient causes parameter displacements invisible to the loss function and thus uncontrollable by the training set, but which can drastically change behaviour away from the training set by destroying the network's covariance structure and inductive bias imbued by initialisation. We can formally describe this pathology by its bias to gradient flow in the infinitesimal limit:

\begin{theorem}
    \label{thm:ridge-biases}
    Consider a wide feature-learning regime network satisfying Assumption~\ref{ass:sigma-min}. Let $\vtheta^\star_{\lambda\norm{\Delta\vtheta}^2}:=\arg \min \limits_{\theta\in\Theta}\Loss{\g{\vtheta}}{Y}+\lambda\norm{\vtheta-\vtheta_0}^2$. Then
    \[
        \vtheta^\star_{\norm{\Delta\vtheta}^2}:=\lim_{\lambda\downarrow0}\vtheta^\star_{\lambda\norm{\Delta\vtheta}^2}\neq\vtheta^\star,
    \]
    where $\vtheta^\star:=\lim \limits_{t\to\infty}\gf{\vtheta_0}$ is the time-limit of unregularised gradient flow.
\end{theorem}

An identical argument proves failure of standard ridge in both kernel \emph{and} feature-learning regimes. Figure~\ref{fig:trajectories} illustrates this pathology exactly. Intuitively, the limiting dynamics has two timescales. On the fast timescale, small-$\lambda$ regularised trajectories stay close to the unregularised one, seemingly converging to $\vtheta^\star$. However, any $\lambda>0$ gives rise to a slow timescale on which the uncontrolled regularisation gradient induces drift fibre directions, biasing the time-limit. For larger $\lambda$ --- necessary in problems with little data or noisy targets --- this pathology becomes immediate, yielding significant divergence from the unregularised trajectory.

While not always a bad thing in practice---this effect is precisely what causes weight decay-induced \emph{grokking}~\citep{power2022grokking,xu2026grok} and underpins the adaptive kernel representer theorem of \citet{lauditi2025adaptive}---in cases where the feature geometry is well-suited to the task a priori (i.e.\ pretrained networks), it can harm generalisation by replacing useful inductive bias with drifted features purely fitting noise~\cite{kumar2022fine}. We expose this pathology empirically in Section~\ref{sec:experiments}, which confirms a massive degradation of test performance under both standard and anchored ridge at high regularisation strengths.

\section{Canonical Regularisation in the Feature-learning Regime}
\label{sec:canonical-regularisation}

Despite its pathology in the feature-learning regime, it is clear that anchored ridge's harmony with gradient flow in the kernel regime is quite special. Indeed, this result implicitly underpins the NN-GP correspondence and Neural Tangent Kernel (NTK) theory as a whole~\cite{jacot2018neural}. It is thus reasonable to describe anchored ridge as the \emph{canonical regulariser} of kernel-regime networks, corresponding to the \emph{function-space canonical prior} NTK Gaussian Process. We thus set out to derive an equivalent regulariser for the feature-learning regime. 

Such a task requires us to return to regime-agnostic first-principles and axioms, ensuring unique identification of anchored ridge in the kernel regime as a validation of our framework. We propose a horizontal function-space energy formulation (with respect to the local Fisher metric $K(\vtheta)$) as a suitably fundamental object underpinning such networks in the context of gradient flow training:

\begin{definition}[Canonical function-space energy, prior, and regulariser]
\label{def:canonical-energy}
For some overparametrised network initialised at $\vtheta_0$, the \emph{canonical function-space energy} at $\vect{c}\in\mathcal{Y}^n$ is the minimum control-theoretic energy required to transport $\g{\vtheta}$ from $\g{\vtheta_0}$ to $\vect{c}$ under gradient flow dynamics:
\begin{align*}
    E(\vect{c}) := &\min_{\vu (t)} \int_0^1 \vu (t)^\top K(t)\,\; \vu (t)\, dt \quad \\ \text{subject to} &\quad \dot\vtheta = \J{\vtheta(t)}^\top \vu ,\ K(t) = \J{\vtheta(t)}\J{\vtheta(t)}^\top,\; \vtheta(0)=\vtheta_0,\; \vtheta(1)=\ginv{\vect{c}}.
\end{align*}
\end{definition}

For kernel-regime networks, the minimum energy control is the constant $\vu^\star(t)=K^{-1}\left(\vect{c}-\g{\vtheta_0}\right)$, recovering the standard squared RKHS norm $E(\vect{c})=\left(\vect{c}-\g{\vtheta_0}\right)^\top K^{-1}\left(\vect{c}-\g{\vtheta_0}\right)\equiv\norm{\vect{c}-\g{\vtheta_0}}_{\mathrm{RKHS}}^2$ (Proposition~\ref{prop:linear-recovery}(c)). 

The canonical function-space prior arises from $E(\vect{c})$ straightforwardly as a Gibbs distribution:

\begin{definition}[Canonical function-space prior]
\label{def:riemannian-gibbs}
    The canonical function-space prior induced by $E$ is a \emph{Riemannian Gibbs Process}, with prediction density on the training set
    \[
        p(\vect{c}\mid X)\propto\exp\!\left(-\beta E(\vect{c})\right).
    \]
    The temperature $\beta$ is the sole degree of freedom under $E$ (due to scale invariance of gradient flow), and thus takes the role of a global prior scale.
\end{definition}

Applying this Riemannian Gibbs Process prior to kernel-regime networks yields the familiar NTK GP, i.e. $\mathcal{GP}\left(\f{\cdot}{\vtheta_0}, (2\beta)^{-1}k_{\mathrm{NTK}}(\cdot,\cdot')\right)$, with $p(\vect{c}\mid X)\propto\exp\!\left(-\beta\norm{\vect{c}-\g{\vtheta_0}}_{\mathrm{RKHS}}^2)\right)$ (Corollary~\ref{cor:kernel-ntk-gp}). Kolmogorov consistency over arbitrary finite test inputs at the fixed training set follows from a parametric pushforward (Proposition~\ref{prop:rgp-finite-test-consistency}); consistency \emph{across} training subsets remains an open problem (Remark~\ref{rmk:rgp-status}).

Completing our generalisation requires one final task: a unique lift of the low-dimensional function-space energy to high-dimensional parameter-space. Reasonable axioms underpin this lift. Namely, we require the lift to be a) additive to the function-space energy $E(\vect{c})$ associated with the training predictions $\g{\vtheta}=\vect{c}$ at $\vtheta$ (thereby forcing the lift to geometrically correspond only to $\fibre{\vect{c}}$, b) quadratic in the ambient Euclidean metric of parameter-space, and c) uniformly $0$ at the gauge-fixed representative $\ginv{\vect{c}}$. This uniquely gives us the \emph{canonical lift} (Theorem~\ref{thm:canonical-lift-uniqueness}):

\begin{definition}[Canonical lift]
    \label{def:canonical-lift}
    The \emph{canonical lift} of $E(\g{\vtheta})$ to parameter-space is
    \[
        R(\vtheta):=E(\g{\vtheta})\;+\;\inf_{\psi}\int \limits_0^1 \norm{\dot{\psi}(t)}^2dt,
    \]
    where the infimum is over piecewise-$C^1$ paths $\psi:[0,1]\to\fibre{\g{\vtheta}}\subset\Theta$ with $\psi(0)=\ginv{\g{\vtheta}}$, $\psi(1)=\vtheta$, and $\g{\psi(t)}=\g{\vtheta}\;\forall\; t\in[0,1]$.
\end{definition}

Verifying this lift in the kernel regime uniquely points to anchored ridge: We have that $E(\g{\vtheta})=\norm{\vect{c}-\g{\vtheta_0}}_{\mathrm{RKHS}}^2\equiv\norm{\Delta\vtheta_{\mathrm{flow}}}^2$ where $\Delta\vtheta_{\mathrm{flow}}$ is the projection of $\vtheta-\vtheta_0$ onto $\M{\vtheta_0}$. Similarly, the lifted term is exactly $\norm{\Delta\vtheta_{\mathrm{fibre}}}^2$, where $\Delta\vtheta_{\mathrm{fibre}}$ is the projection of $\vtheta-\vtheta_0$ onto $\fibre{\g{\vtheta}}$. As in the kernel regime these two sets are orthogonal and affine, the Pythagorean identity yields $R(\vtheta)=\norm{\Delta\vtheta_{\mathrm{flow}}}^2+\norm{\Delta\vtheta_{\mathrm{fibre}}}^2=\norm{\vtheta-\vtheta_0}^2$, uniquely recovering anchored ridge (Corollary~\ref{cor:kernel-anchored-ridge}; Proposition~\ref{prop:linear-recovery}(d)).

Framework in hand, we now present the canonical regulariser of the feature-learning regime:

\begin{proposition}[Geodesic ridge]
    \label{prop:geodesic-ridge}
    The canonical regulariser of feature-learning regime networks is \emph{geodesic ridge}, defined as
    \[
        \dist{2}{}{\vtheta}{\vtheta_0}:=\dist{2}{\mathrm{flow}}{\ginv{\g{\vtheta}}}{\vtheta_0}+\dist{2}{\mathrm{fibre}}{\vtheta}{\ginv{\g{\vtheta}}},
    \]
    where $\dist{2}{\mathrm{flow}}{\ginv{\g{\vtheta}}}{\vtheta_0}\equiv E(\g{\vtheta})$ is the minimum horizontal distance in parameter-space from $\vtheta_0$ to the the gauge-fixed representative $\ginv{\g{\vtheta}}$, and $\dist{2}{\mathrm{fibre}}{\vtheta}{\ginv{\g{\vtheta}}}$ is the geodesic distance from the gauge-fixed representative to $\vtheta$ along the output fibre $\fibre{\g{\vtheta}}$.
\end{proposition}

While generally irreducible beyond this form, geodesic ridge is theoretically complete:

\begin{theorem}[Canonical regulariser characterisation of geodesic ridge]
\label{thm:canonical-characterisation}
Under Assumption~\ref{ass:sigma-min}:
\begin{enumerate}[label=(\alph*),nosep]
    \item \textbf{Unique minimiser on each fibre.} For every $\vect{c} \in C_0(\vtheta_0)$, $\ginv{\vect{c}}$ is the unique minimiser of $\dist{2}{}{\cdot}{\vtheta_0}$ over $\fibre{\vect{c}} \cap V$.
    \item \textbf{Trajectory-level minimality.} Once the gradient flow trajectory $\vtheta(t)$ enters $V$, $\vtheta(t) = \vtheta'(\g{\vtheta(t)})$ and $\dist{2}{}{\vtheta(t)}{\vtheta_0} = E(\g{\vtheta(t)})$. In particular every intermediate trajectory point is the minimum-cost representative of its output vector.
    \item \textbf{Unbiased limit property.} $\lim \limits_{\lambda \downarrow 0} \Loss{\g{\vtheta}}{Y} + \lambda \dist{2}{}{\vtheta}{\vtheta_0}=\vtheta^\star$.
\end{enumerate}
\end{theorem}

These properties are exemplified in Figure~\ref{fig:trajectories}. We have thus generalised anchored ridge to the feature-learning regime, theoretically alleviating all issues associated with anchored ridge and thus allowing for principled modelling of observation noise and uncertainty in the feature-learning regime.

\subsection{Arc Ridge: A Minimax-Robust Upper Bound to Geodesic Ridge}
\label{subsec:arc-ridge}

Geodesic ridge, while theoretically complete, is intractable at scale. However, two observations combine to yield a practical surrogate. First, under exact gradient flow trajectories remain on $\M{\vtheta_0}$ by construction, so the fibre term in Definition~\ref{prop:geodesic-ridge} vanishes. Second, geodesic ridge is upper-bounded by the squared length of \emph{any} on-manifold path, \textbf{including the realised training path}:

\begin{definition}[Squared path length regulariser (arc ridge)]
\label{def:path-length}
For a training trajectory $\gamma : [0,T] \to \Real^m$, the \emph{squared path length regulariser}, or \emph{arc ridge}, is
\[
    L^2(\gamma) \;:=\; \left(\int_0^T \norm{\dot\gamma(t)}\, dt\right)^2.
\]
$L^2(\gamma)$ depends only on the image of $\gamma$ as a subset of $\Theta$, not on its time-parametrisation.
\end{definition}

\begin{theorem}[Squared path length upper-bounds the canonical regulariser]
\label{thm:path-length-upper}
Let $\gamma_{\text{train}} : [0,T] \to \M{\vtheta_0}$ be any piecewise-$C^1$ path with $\gamma_{\text{train}}(0) = \vtheta_0$ and $\gamma_{\text{train}}(T) = \vtheta$. Then
\[
    \dist{2}{\mathrm{flow}}{\vtheta_0}{\vtheta} \;\le\; L^2(\gamma_{\text{train}});
\]
under exact gradient flow (where the fibre term vanishes), $\dist{2}{}{\vtheta}{\vtheta_0} \le L^2(\gamma_{\text{train}})$. The endpoint gradient $\nabla_{\vtheta}L^2$ lies along $\dot\gamma_{\text{train}}(T^-) \in \mathrm{Im}(\J{\vtheta}^\top) \cap T_{\vtheta}\M{\vtheta_0}$ (in particular $\nabla_{\vtheta}L^2 \in \ker N(\g{\vtheta})$), so regularisation by $L^2$ introduces no component along the output fibres or indeed orthogonal to the realised gradient flow direction. Proof deferred to the appendix.
\end{theorem}

The bound follows from the definition of geodesic distance as an infimum of arc lengths chained with the energy--length identity (Proposition~\ref{prop:constant-speed}); proof deferred to the appendix. Anchored ridge trivially lower-bounds geodesic ridge on $\M{\vtheta_0}$ as the solution to a relaxation of the same defining problem (Proposition~\ref{prop:ridge-lower}); the resulting sandwich and the shared vanishing-regularisation limit of all three regularisers at $\vtheta^\star$ are recorded in Corollary~\ref{cor:sandwich}.

In the absence of tractability of the canonical regulariser, it is principled to instead solve a minimax-robust relaxation of MAP training, choosing the regulariser that minimises worst-case regret within the sandwich $\norm{\vtheta-\vtheta_0}^2 \le \dist{2}{}{\vtheta}{\vtheta_0} \le L^2(\gamma_{\text{train}})$. This directly motivates arc ridge regularisation. Furthermore, arc ridge requires storing only the accumulated arc length---a single float---since the regularisation gradient direction is fully determined by the current data gradient. It is therefore applicable at scale, and is used for most of our empirical results.

Arc ridge has a further theoretical property: it connects to early stopping, since regularisation by accumulated path length behaves like a time-based stopping criterion and, given an estimate of observation noise, obviates the need for a validation set (Proposition~\ref{prop:identical-trajectory}). The relationship between ridge regularisation and early stopping is well-studied, and in agreement with our theory, early stopping generalises trivially to feature-learning regime networks~\cite{engl1996regularization,yao2007early}, unlike anchored ridge.

\section{Empirical Consequences}
\label{sec:experiments}

\vspace{-0.2cm}

\begin{figure}
  \centering
  \begin{minipage}{0.49\textwidth}
    \centering
    \includegraphics[width=\linewidth,trim={0.35cm 0.4cm 0.35cm 0.35cm},clip]{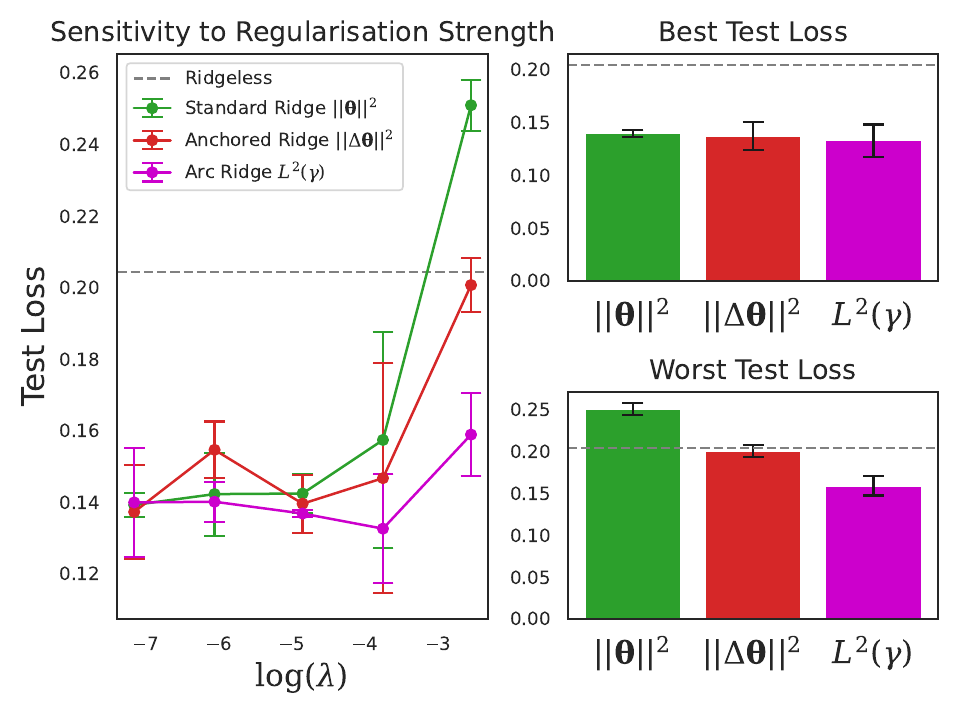}
\end{minipage}
\hfill
\begin{minipage}{0.49\textwidth}
    \centering
    \includegraphics[width=\linewidth,trim={0.35cm 0.4cm 0.35cm 0.35cm},clip]{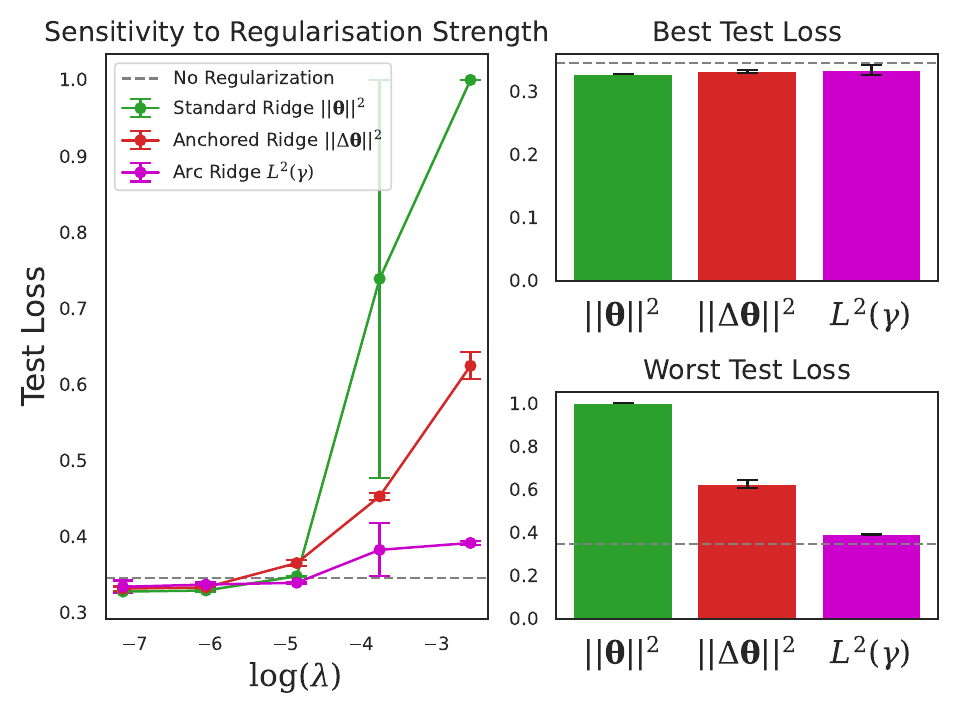}
\end{minipage}
  \caption{Test set MSE on UTKFace (left) and Yelp Review (right) for standard, anchored, and arc ridge across regularisation strengths (with early stopping), against an unregularised baseline (dashed). At high $\lambda$, standard and anchored ridge degrade sharply as their off-manifold gradient components destroy the pretrained prior; arc ridge degrades only mildly, consistent with benign undertraining. At low $\lambda$, early stopping masks the bias and all regularisers converge --- corroborating the formal equivalence between arc ridge and early stopping under gradient flow (Proposition~\ref{prop:identical-trajectory}).}
  \label{fig:main_experiments}
  \vspace{-0.3cm}
\end{figure}

We first demonstrate the consequences of Theorems~\ref{thm:ridge-biases} and~\ref{thm:canonical-characterisation} in Figure~\ref{fig:trajectories}, which displays gradient flow trajectories on the minimal example for each regime. Figure~\ref{fig:main_experiments} then demonstrates the practical implications of our theory on realistic transfer-learning problems requiring feature-learning: Age regression on the UTKFace datatset~\citep{zhifei2017cvpr} transferring from ResNet18~\citep{he2016deep}, and review score regression from the Yelp Review dataset~\citep{zhang2015character} transferring from DistilBERT~\citep{sanh2019distilbert}. We present results for each regulariser across a range of different regularisation strengths, and combine with early stopping for maximum realism. To demonstrate the necessity of regularisation, we also include the ''Ridgeless`` (unregularised) performance (with no early stopping) as a baseline. Our theory predicts that at high regularisation strengths, the pathological gradient component of standard and anchored ridge will cause significant degradation of the prior imbued by pretraining, manifesting as very poor test performance. Conversely, arc ridge should only lose test performance in over-regularised regimes commensurately with undertraining --- a much more controlled phenomenon --- and thus exhibit a milder degradation. In contrast, at lower regularisation strengths the pathological components of anchored and standard ridge are dominated by early stopping --- equivalent to optimally-tuned arc ridge --- and thus all regularisers should converge in performance. This is exactly what we observe. 

\vspace{-0.2cm}

\section{Conclusion and Limitations}

\vspace{-0.2cm}

We proved that the implicit prior of a wide network is governed not by any norm on parameters, but by the geometry of the trajectories gradient flow can trace. We showed that while in the kernel regime anchored ridge is geometrically aligned with training, in the feature-learning regime training geometry curves, causing the regularisation gradient of anchored (and standard) ridge to pathologically degrade the implicit prior at initialisation --- a critical failure mode in transfer learning, for example. We proposed a correction to this by identifying two axioms uniquely identifying anchored in the kernel regime, and applied these to the curved geometry of feature-learning to obtain a principled framework not only yielding \emph{geodesic ridge} as a generalisation of anchored ridge to the feature-learning regime, but also revealing the implicit prior here as a Riemannian Gibbs process. For practical unbiased regularisation at scale, we proposed \emph{arc ridge} a minimax-robust surrogate to true geodesic ridge, which is cheap to compute at scale and reveals deep links to early stopping. Finally, we validated the predictions empirically: regularisers that respect the geometry of training preserve the implicit prior; those that ignore it do not.

Several directions remain open. Empirically, tighter scalable approximations to geodesic ridge, broader validation across architectures and tasks, and rigorous verification of Assumption~\ref{ass:sigma-min} in concrete settings are immediate next steps. Theoretically, full Kolmogorov consistency of the Riemannian Gibbs Process across training subsets, analysis of the flow manifold in the $\mu$P infinite-width limit, and extensions to non-interpolating regimes \cite{soudry2018implicit, lyu2020gradient} and preconditioned optimisers such as Adam \cite{kingma2014adam} or Muon \cite{jordan2024muon} all remain open.

\bibliographystyle{plainnat}
\bibliography{references}


\appendix

\section{Assumptions}
\label{sec:assumptions}

We collect here the full list of \emph{atomic} hypotheses on which the main results rely. Each entry states the assumption, lists the main-text results that depend on it, and sketches its justification. Smoothness of the flow manifold near the interpolant, normal hyperbolicity of $\Theta^\star$, local properness of $g|_{\M{\vtheta_0}}$, and existence of a unique gradient flow $\omega$-limit are derivable consequences --- stated and proved as theorems or lemmas in Appendices~\ref{sec:lemmas} and~\ref{sec:proofs} rather than assumed.

\paragraph{(A.1) Smoothness.}
\emph{Statement.} The network $\f{\cdot}{\vtheta}$ is $C^\infty$ in $\vtheta$ on the parameter domain of interest.
\emph{Used by.} All main-text results.
\emph{Justification.} Finite composition of smooth activations is $C^\infty$ by elementary calculus. The practically relevant restriction is that activations satisfy the pseudo-Lipschitz / polynomially bounded growth hypothesis imposed by the TensorPrograms framework \citep{yang2020tensor,yang2020tensor3}, covering analytic variants of $\tanh$, smooth $\mathrm{SiLU}$, and smooth approximations of $\mathrm{ReLU}$.

\paragraph{(A.2) Data.}
\emph{Statement.} The training dataset $\mathcal{D} = (X, Y)$ has finite $n$ and distinct inputs $X_i \ne X_j$ for $i \ne j$.
\emph{Used by.} All main-text results.
\emph{Justification.} Distinct inputs are necessary for the output fibre $\fibre{\vect{c}}$ to be well-defined at every achievable $\vect{c}$; finiteness of $n$ ensures the Jacobian $\J{\vtheta} \in \Real^{n\times m}$ has finite domain.

\paragraph{(A.3) Loss and gradient flow dynamics.}
\emph{Statement.} Training is the continuous-time gradient flow $\dot\vtheta = -\nabla_{\vtheta} \Loss{\g{\vtheta}}{Y} = \J{\vtheta}^\top \vu $ with $\vu  := -\partial_{\vect{c}}\Loss{\vect{c}}{Y}|_{\vect{c}=\g{\vtheta}}$. The main text takes $\Loss{\vect{c}}{Y} = \norm{\vect{c}-Y}^2$ (squared error), giving $\vu  = 2(Y - \g{\vtheta})$, $\nabla^2_{\vect{c}}\Loss{\cdot}{Y} = 2I$, and Hessian $\nabla^2_{\vtheta}\Loss{\g{\vtheta}}{Y} = 2\,\J{\vtheta}^\top\J{\vtheta} + 2\sum_i (\g{\vtheta}-Y)_i\, \nabla^2 g_i$, which reduces to $2\,\J{\vtheta}^\top\J{\vtheta}$ at any $\vtheta\in\Theta^\star$.
\emph{Used by.} All main-text results.
\emph{Justification.} A definitional choice of optimiser. Discrete-time gradient descent approximates this dynamics in the small step-size limit. The squared-error specialisation is natural for the canonical regression setting underlying the NTK--GP correspondence; extension to any strictly convex, $\mu$-strongly-convex, $M$-smooth loss is straightforward (the SSE-specific identity Hessian becomes $\mu I \preceq \nabla^2_{\vect{c}}\Loss{}{Y} \preceq M I$ everywhere needed) and is treated implicitly throughout the appendix.

\paragraph{(A.4) Uniform Jacobian conditioning near interpolation.}
\emph{Statement.} There exist an open neighbourhood $U \supseteq \Theta^\star$ and constants $0 < c \le C < \infty$ such that for all $\vtheta \in U$,
\[
    c \;\le\; \sigma_{\min}(\J{\vtheta}) \quad\text{and}\quad \sigma_{\max}(\J{\vtheta}) \;\le\; C.
\]
The gradient flow trajectory from $\vtheta_0$ converges to a limit point $\vtheta^\star \in \Theta^\star \cap U$ (formally Assumption~\ref{ass:sigma-min}).
\emph{Used by.} All main-text Theorems and Definitions in Sections~\ref{sec:ridge-pathology} and~\ref{sec:canonical-regularisation}; Lemmas~\ref{lem:fibre-sub},~\ref{lem:nhim-stable},~\ref{lem:traj-length},~\ref{lem:properness}; and proofs throughout Appendix~\ref{sec:proofs}.
\emph{Justification.} A.4 is a single hypothesis collecting the geometric structure required for the local-NHIM analysis. Its three components --- lower spectral bound on $\J{\vtheta}$, upper spectral bound, and convergence of the gradient flow trajectory to $\Theta^\star$ --- are all standard in wide-network analysis. Under $\mu P$ scaling for TensorPrograms-compatible architectures with smooth activations \citep{yang2020tensor,yang2020tensor3}, $\sigma_{\min}(\J{\vtheta_0})$ and $\sigma_{\max}(\J{\vtheta_0})$ concentrate around positive bounded infinite-width limits at initialisation; continuity propagates these bounds to a neighbourhood. Convergence of gradient flow to an interpolator is observed empirically and proved in special cases \citep{jacot2018neural,du2019gradient}. Three downstream properties are derived (not assumed) from A.4: $\Theta^\star \cap U$ is an embedded $C^\infty$ codimension-$n$ submanifold (regular-value theorem applied to $g$~\citep{lee2003smooth}); the loss Hessian $\nabla^2 \mathcal{L} = 2J^\top J$ on $\Theta^\star \cap U$ has $n$ eigenvalues bounded below by $2c^2$ and $m-n$ zero eigenvalues with eigenspace $\ker\J{\vtheta}$, making $\Theta^\star \cap U$ normally hyperbolic for the gradient flow vector field; and the leaf $\M{\vtheta_0}\cap V$ of the resulting stable-manifold foliation, restricted to a smaller neighbourhood $V\subseteq U$, is an embedded $n$-dimensional $C^\infty$ submanifold tangent to $\mathrm{Im}(\J{\vtheta^\star}^\top)$ at $\vtheta^\star$ (Theorem~\ref{thm:local-chart}).

\begin{remark}[$\mu P$ scaling and Assumption~\ref{ass:sigma-min}]
\label{rmk:mup-reduction}
Under $\mu P$ scaling for TensorPrograms-compatible architectures with smooth activations \citep{yang2020tensor,yang2020tensor3}, the NTK Gram matrix $K(\vtheta_0) = \J{\vtheta_0}\J{\vtheta_0}^\top$ concentrates around its positive-definite infinite-width limit at initialisation, with eigenvalues bounded above and below at width $m\to\infty$. Continuity of $\sigma_{\min}, \sigma_{\max}$ in $\vtheta$ then propagates the spectral bounds to a neighbourhood of $\vtheta_0$, and stability arguments along the gradient flow trajectory extend them to a neighbourhood of $\Theta^\star$.
\end{remark}


\section{Extended Related Work}
\label{app:extended-related-work}

\paragraph{Implicit bias and minimum-norm interpolation.}
The connection between gradient flow and minimum-norm solutions in overparametrised linear models is classical~\citep{bartlett2020benign}: the gradient flow limit is the Moore--Penrose pseudoinverse solution, which is also the $\ell_2$-minimum-norm interpolator. The same identity holds in kernel regression and in NTK-regime neural networks~\citep{jacot2018neural,lee2019wide,du2019gradient}: the gradient flow limit is the minimum-RKHS-norm interpolator, whose Gibbs distribution is the Gaussian Process prior $\mathcal{N}(\g{\vtheta_0}, K)$~\citep{rasmussen2006gaussian,he2020bayesian,ordonez2026gaussian}. The \emph{kernel vs.\ rich} dichotomy of \citet{woodworth2020kernel} delineates when networks behave like linear models (kernel regime) vs.\ when feature learning occurs. Our contribution is to identify the \emph{energy} underlying this selection---not merely the norm---and to generalise it to the feature-learning regime. Theorem~\ref{thm:ridge-biases} shows that, beyond the kernel regime, the min-Euclidean-norm interpolator and the gradient flow limit diverge; the energy-based view moves beyond the implicit-bias framing to a positive characterisation of the canonical function-space prior.

\paragraph{Implicit bias in classification and structured models.}
The implicit bias of gradient descent has been characterised in several other settings. For classification with the logistic loss on linearly separable data, gradient descent converges in direction to the max-margin (hard-SVM) solution~\citep{soudry2018implicit,ji2019implicit}; for homogeneous networks this extends to a KKT-based characterisation of the margin~\citep{lyu2020gradient}. For matrix factorisation, gradient descent induces an implicit nuclear-norm penalty~\citep{gunasekar2017implicit}, and for wide two-layer networks trained with the logistic loss a related max-margin structure emerges~\citep{chizat2020implicit}. These results characterise the \emph{direction} of implicit bias in specific architectures or losses; our framework instead characterises the \emph{function-space energy} underlying gradient flow selection and provides a geometric generalisation that applies across architectures and losses within the feature-learning regime.

\paragraph{Feature-learning theory.}
The mean-field limit~\citep{chizat2018global,mei2018mean,rotskoff2018neural} describes the infinite-particle dynamics of two-layer networks in a regime where the Jacobian varies substantially, complementing the NTK picture. The Tensor Programs framework of~\citet{yang2021tensor} provides a unified characterisation of wide networks via a dynamical dichotomy between kernel and feature-learning regimes, with $\mu P$ scaling~\citep{yang2021tuning} as the canonical parameterisation enabling stable feature learning. Recent work has analysed what gradient descent actually \emph{learns} in the feature-learning regime: \citet{damian2022neural} prove that gradient descent can learn representations unavailable to fixed-kernel methods, and \citet{lauditi2025adaptive} characterise the adaptive kernel that emerges at the feature-learning infinite-width limit. Our work complements these results by characterising the \emph{regulariser} that is canonical for the dynamics induced by gradient flow in this regime, rather than describing what those dynamics converge to.

\paragraph{Weight decay beyond RKHS regularisation.}
Weight decay has been studied as a driver of phenomena beyond its interpretation as an RKHS prior. The grokking phenomenon---delayed generalisation long after data loss reaches zero---was first documented by~\citet{power2022grokking} and has been given a provable analysis in the setting of ridge regression by~\citet{xu2026grok}. The effect of weight decay on the features learned by neural networks, and its connection to the adaptive kernel at the feature-learning limit, is analysed by~\citet{lauditi2025adaptive}. Theorem~\ref{thm:ridge-biases} and Remark~\ref{rmk:weight-decay-fails} provide a complementary theoretical perspective: in the feature-learning regime, both standard ridge and anchored ridge select interpolating solutions that differ systematically from the gradient flow limit. Their practical effect is therefore \emph{solution-set modification} rather than \emph{prior specification}, and the benefits observed in practice may arise partly from this solution-selection property rather than from Bayesian regularisation.

\paragraph{Geometry-aware optimisers.}
Natural gradient descent~\citep{amari1998natural} preconditions the gradient by the inverse Fisher information matrix, which in the supervised learning setting equals $\K{\vtheta}^{-1}$~\citep{martens2020new,ollivier2015riemannian}. Under this preconditioning, function-space dynamics become isotropic and training trajectories approximate geodesics on $\M{\vtheta_0}$~\citep{zhang2019fast,cai2019gram}; several wide-network approximations reproduce this isotropic behaviour~\citep{karakida2020understanding}. Recent optimisers implement related principles at scale~\citep{jordan2024muon,bernstein2024old}. Our work is complementary: rather than proposing a new optimiser, we characterise the \emph{regulariser} canonical for the geometry induced by gradient flow; preconditioning the gradient flow by the natural-gradient (NTK) metric is a direct route to tightening our path-length bound, but we leave its full development to future work.


\section{Extended Results}
\label{app:extended-results}

We collect here the formal statements of results whose informal versions appear in the main text. Proofs of all results appear in Appendix~\ref{sec:proofs}.

\subsection{Convergence of gradient flow under uniform Jacobian conditioning}
\label{subsec:gf-convergence}

\begin{proposition}[Gradient-flow convergence to an interpolator]
\label{prop:gf-convergence}
Under Assumptions A.1, A.3, and~\ref{ass:sigma-min}, the gradient flow trajectory $\vtheta(t)=\gf{\vtheta_0}$ converges to a limit $\vtheta^\star\in\Theta^\star$ at exponential rate. For squared-error loss, $\norm{\g{\vtheta(t)} - Y} \le \norm{\g{\vtheta_0} - Y}\,e^{-2c^2 t}$, where $c$ is the lower spectral bound of Assumption~\ref{ass:sigma-min}. In particular, $\vtheta(t)$ stays in $U$ for all $t\ge 0$, and $\omega(\vtheta_0):=\lim_{t\to\infty}\vtheta(t)$ exists.
\end{proposition}

\subsection{Local manifold and metric structure of $\M{\theta}$}
\label{subsec:local-chart}

\begin{theorem}[Local manifold chart]
\label{thm:local-chart}
Under Assumption~\ref{ass:sigma-min}, there exists an open neighbourhood $V\subseteq U$ of $\vtheta^\star$ such that:
\begin{enumerate}[label=(\roman*),nosep]
    \item $\Theta^\star \cap V$ is an embedded $C^\infty$ codimension-$n$ submanifold of $V$ with $T_{\vtheta}(\Theta^\star\cap V) = \ker\J{\vtheta}$ at every $\vtheta\in\Theta^\star\cap V$;
    \item the gauge-fixed flow manifold $\M{\vtheta_0}\cap V$ is an embedded $C^\infty$ $n$-submanifold of $V$ tangent at $\vtheta^\star$ to $\mathrm{Im}(\J{\vtheta^\star}^\top)$, varying $C^\infty$-smoothly with the basepoint across $\Theta^\star\cap V$;
    \item the restriction $g|_{\M{\vtheta_0}\cap V}$ is a $C^\infty$ diffeomorphism onto its open image $C_0(\vtheta_0):=g(\M{\vtheta_0}\cap V)\subseteq\Real^n$, with inverse the gauge-fixed representative $\ginv{\vect{c}}$.
\end{enumerate}
\end{theorem}

\begin{theorem}[Metric gap of $\M{\vtheta_0}$ from the kernel-regime affine plane]
\label{thm:metric-gap}
Let $A(\vect{c}) := D\ginv{\vect{c}}$ and write $A(\vect{c}) = A_\parallel(\vect{c}) + N(\vect{c})$ with $A_\parallel(\vect{c}) := \J{\ginv{\vect{c}}}^\top \K{\ginv{\vect{c}}}^{-1}$ and $N(\vect{c}) := (I - P_{\ginv{\vect{c}}})\,A(\vect{c})$, where $P_{\vtheta}$ is the orthogonal projector onto $\mathrm{Im}(\J{\vtheta}^\top)$. Then the pullback metric $G(\vect{c}) := A(\vect{c})^\top A(\vect{c})$ on $C_0(\vtheta_0)$ admits the orthogonal decomposition
\[
    G(\vect{c}) \;=\; \K{\ginv{\vect{c}}}^{-1} \;+\; N(\vect{c})^\top N(\vect{c}),
\]
so $G \succeq K^{-1}$ pointwise, with equality $G(\vect{c}) = \K{\ginv{\vect{c}}}^{-1}$ if and only if $T_{\ginv{\vect{c}}}\M{\vtheta_0} = \mathrm{Im}(\J{\ginv{\vect{c}}}^\top)$ (the kernel-regime alignment). Moreover, along the realised gradient flow trajectory $\vtheta(t)=\ginv{\vect{c}(t)}$, the velocity $\dot{\vect{c}}(t)$ satisfies $N(\vect{c}(t))\,\dot{\vect{c}}(t) = 0$, so the trajectory's $G$-energy reduces to $\dot{\vect{c}}^\top \K{\vtheta(t)}^{-1}\dot{\vect{c}}$.
\end{theorem}

\subsection{Uniqueness of the canonical kernel-regime energy and lift}

The kernel-regime canonical function-space energy $E(\vect{c}) = (\vect{c}-\g{\vtheta_0})^\top K^{-1}(\vect{c}-\g{\vtheta_0})$ of Definition~\ref{def:canonical-energy} is the unique continuous output-space energy compatible with the natural geometry of the flow manifold; ``naturality'' alone --- the property that vanishing-regularisation MAP estimation recovers the gradient flow limit --- does not single it out. We make these statements formal here.

\begin{theorem}[Output-space energy uniqueness]
\label{thm:output-energy-uniqueness}
Let $G \in \Real^{n\times n}$ be symmetric positive definite. Suppose $\rho_G : \Real^n \to [0,\infty)$ is continuous, satisfies $\rho_G(\vect{0}) = 0$, is nontrivial, and obeys
\begin{enumerate}[label=(\roman*),nosep]
    \item ($G^{-1}$-isometry invariance) $\rho_G(U\vect{c}) = \rho_G(\vect{c})$ for every $U \in \Real^{n\times n}$ with $U^\top G^{-1} U = G^{-1}$;
    \item (orthogonal additivity) $\vu ^\top G^{-1} \vect{v} = 0 \implies \rho_G(\vu +\vect{v}) = \rho_G(\vu ) + \rho_G(\vect{v})$.
\end{enumerate}
Then $\rho_G(\vect{c}) = \alpha\, \vect{c}^\top G^{-1} \vect{c}$ for some constant $\alpha > 0$.
\end{theorem}

In words, axiom (i) asks that $\rho_G$ depends only on the $G^{-1}$-geometry of the displacement, not on the basis we describe it in (the analogue of rotational invariance for the standard inner product); axiom (ii) asks that $G^{-1}$-orthogonal displacements contribute additively (the function-space analogue of the Pythagorean identity). Together they leave a single scale degree of freedom, absorbed into the prior temperature $\beta$. Applied with $G = K^{-1}$, this characterises the kernel-regime canonical energy of Definition~\ref{def:canonical-energy} uniquely. The corresponding parameter-space lift inherits this uniqueness: in the kernel regime, the orthogonal decomposition $\Real^m = \mathrm{Im}(J^\top)\oplus \ker J$ coincides with the flow-tangent / fibre decomposition globally, and the squared ambient distance along $\fibre{\g{\vtheta}}$ is the unique fibre extension respecting orthogonal additivity in this decomposition.

\begin{remark}[Both axioms of Theorem~\ref{thm:output-energy-uniqueness} are individually necessary]
\label{rmk:axioms-necessary}
Dropping axiom (i) (isometry invariance), the functional $\rho_G(\vect{c}) := (\vect{c}_1 - \g{\vtheta_0}_1)^2$ is continuous, vanishes at the origin, is non-trivial, and is orthogonally additive (every basis-aligned splitting of $\vect{c}$ trivially satisfies the additivity hypothesis); but it depends on the chosen first coordinate and is not of the form $\alpha\,\vect{c}^\top G^{-1}\vect{c}$. Dropping axiom (ii) (orthogonal additivity), the functional $\rho_G(\vect{c}) := \big(\vect{c}^\top G^{-1}\vect{c}\big)^2$ is continuous, $G^{-1}$-isometry invariant, and non-trivial; but $G^{-1}$-orthogonal displacements satisfy $\rho_G(\vu +\vect{v}) = (s+t)^2 \neq s^2+t^2 = \rho_G(\vu )+\rho_G(\vect{v})$ for $s,t>0$, so axiom (ii) fails. Hence both axioms are individually necessary for the uniqueness conclusion.
\end{remark}

\begin{proposition}[No-go for naturality alone in the kernel regime]
\label{prop:naturality-nogo}
Let $\vtheta^\star$ be the gradient-flow interpolator from $\vtheta_0$ in a kernel-regime model with constant Jacobian $J\in\Real^{n\times m}$ and $K=JJ^\top\succ 0$. Consider the two-parameter family of trajectory-blind quadratic regularisers
\[
    \mathcal{F} \;:=\; \Bigl\{\,R_{a,B}(\vtheta) \;:=\; a\,\norm{\vtheta-\vtheta_0}^2 \;+\; \bigl(J(\vtheta-\vtheta_0)\bigr)^\top B\,\bigl(J(\vtheta-\vtheta_0)\bigr)
        \;:\; a > 0,\ B\in\mathrm{Sym}^+_n\,\Bigr\},
\]
where $\mathrm{Sym}^+_n$ denotes the cone of symmetric positive semi-definite $n\times n$ matrices. Distinct $(a,B)$ give distinct regularisers in $\mathcal{F}$ (their Hessians $2(aI_m + J^\top B J)$ differ), and distinct elements of $\mathcal{F}$ correspond to distinct Gaussian MAP parameter priors $\mathcal{N}(\vtheta_0, \tfrac{1}{2\lambda}(aI_m + J^\top B J)^{-1})$. Yet every element of $\mathcal{F}$ attains the same vanishing-regularisation limit:
\[
    \lim_{\lambda\downarrow 0}\,\arg\min_{\vtheta}\bigl[\Loss{\g{\vtheta}}{Y} + \lambda R_{a,B}(\vtheta)\bigr] \;=\; \vtheta^\star \qquad \forall (a,B)\in\Real_{>0}\times\mathrm{Sym}^+_n.
\]
The vanishing-regularisation property (``naturality'') therefore does \emph{not} pin out anchored ridge --- only the special case $B=0$ does, and singling it out within $\mathcal{F}$ requires the additional orthogonal-additivity axiom of Theorem~\ref{thm:output-energy-uniqueness}.
\end{proposition}

\begin{theorem}[Uniqueness of the canonical lift]
\label{thm:canonical-lift-uniqueness}
Fix the canonical function-space energy $E$ of Definition~\ref{def:canonical-energy}. Among all functionals $R:\Theta\to[0,\infty)$ satisfying
\begin{enumerate}[label=(\alph*),nosep]
    \item (additivity to $E$) $R(\vtheta) = E(\g{\vtheta}) + S(\vtheta)$ for some $S:\Theta\to[0,\infty)$ depending only on the on-fibre displacement of $\vtheta$ from $\ginv{\g{\vtheta}}$;
    \item (Euclidean fibre quadratic) the restriction of $S$ to each fibre $\fibre{\g{\vtheta}}\cap V$ is a quadratic in the ambient Euclidean metric inherited from $\Real^m$;
    \item (vanishing at the gauge-fixed representative) $S(\ginv{\g{\vtheta}})=0$,
\end{enumerate}
the unique solution (up to a positive scale) is
\[
    R(\vtheta) \;=\; E(\g{\vtheta}) \;+\; \inf_{\psi}\int_0^1 \norm{\dot\psi(t)}^2\,dt
    \;=\; E(\g{\vtheta}) \;+\; \dist{2}{\mathrm{fibre}}{\ginv{\g{\vtheta}}}{\vtheta},
\]
where the infimum runs over piecewise-$C^1$ paths $\psi:[0,1]\to\fibre{\g{\vtheta}}$ with $\psi(0)=\ginv{\g{\vtheta}}$, $\psi(1)=\vtheta$. In the kernel regime this reduces to anchored ridge; in the feature-learning regime it is geodesic ridge of Proposition~\ref{prop:geodesic-ridge}.
\end{theorem}

\begin{theorem}[Combined uniqueness of geodesic ridge as the canonical regulariser]
\label{thm:geodesic-ridge-uniqueness}
Under Assumption~\ref{ass:sigma-min}, the only functional $R : \Theta\to[0,\infty)$ on a neighbourhood of $\Theta^\star$ which simultaneously satisfies the function-space-energy axioms of Theorem~\ref{thm:output-energy-uniqueness} (with $G=K^{-1}$ at $\vtheta_0$) on the output channel $\g{\vtheta}$, the lift axioms (a)--(c) of Theorem~\ref{thm:canonical-lift-uniqueness} on the parameter-space channel, and the vanishing-regularisation property $\arg\min(\Loss{\g{\cdot}}{Y}+\lambda R)\to\vtheta^\star$ as $\lambda\downarrow 0$, is, up to a single positive scale,
\[
    R(\vtheta) \;=\; \dist{2}{}{\vtheta}{\vtheta_0} \;=\; \dist{2}{\mathrm{flow}}{\vtheta_0}{\ginv{\g{\vtheta}}} + \dist{2}{\mathrm{fibre}}{\ginv{\g{\vtheta}}}{\vtheta},
\]
i.e.\ geodesic ridge. In the kernel regime this collapses to anchored ridge; the corresponding function-space prior is uniquely the Riemannian Gibbs Process of Definition~\ref{def:riemannian-gibbs}, collapsing to the NTK Gaussian process under linear $J$.
\end{theorem}

\subsection{Geometry of the flow manifold and output fibres}

\begin{proposition}[Output fibres are smooth submanifolds]
\label{prop:fibre-submanifold}
Under Assumption~\ref{ass:sigma-min}, for every $\vect{c} \in g(U)$ the output fibre $\fibre{\vect{c}} \cap U$ is an embedded $C^\infty$ submanifold of $\Real^m$ of codimension $n$.
\end{proposition}

\begin{remark}[Path-connectivity of output fibres]
\label{rmk:fibre-connected}
In the overparametrised regime $m \gg n$, the connected component of $\fibre{\vect{c}} \cap V$ containing $\ginv{\vect{c}}$ is path-connected via the implicit-function-theorem chart on $\fibre{\vect{c}}$ over $V$~\citep{lee2003smooth}, so the fibre distance $\dist{}{\mathrm{fibre}}{\vtheta}{\ginv{\vect{c}}} = \norm{\vtheta - \ginv{\vect{c}}}$ in Definition~\ref{def:canonical-lift} is finite.
\end{remark}

\begin{corollary}[Unique fibre--flow intersection in $V$]
\label{cor:intersection-basin}
Under the hypotheses of Theorem~\ref{thm:local-chart}, for every $\vect{c} \in C_0(\vtheta_0)$ the intersection $(\M{\vtheta_0} \cap V) \cap \fibre{\vect{c}}$ consists of exactly one point, which we denote $\ginv{\vect{c}}$. In particular $\vtheta^\star = \vtheta'(Y)$, and the $\omega$-limit $\omega(\vtheta_0) = \lim_{t\to\infty}\gf{\vtheta_0} = \vtheta^\star$.
\end{corollary}

\begin{proposition}[Constant-speed optimality]
\label{prop:constant-speed}
Among piecewise-$C^1$ paths $\gamma : [0,1] \to (\M{\vtheta_0},\,\text{ambient Euclidean})$ with fixed endpoints, the squared-energy $\int_0^1 \norm{\dot\gamma(t)}^2\,dt$ is minimised by constant-speed parametrisations of length-minimising paths, and at the minimiser the energy equals the squared arc length: $E(\gamma_{\mathrm{opt}}) = L^2(\gamma_{\mathrm{opt}})$.
\end{proposition}

\begin{remark}[Existence and non-uniqueness of minimising geodesics]
\label{rmk:geodesic-existence}
On the smooth Riemannian manifold $(\M{\vtheta_0}\cap V, g_{\text{ambient}})$ near $\vtheta^\star$, geodesic completeness on a compact sublevel set of the squared distance from $\vtheta_0$ follows from local properness of $g|_{\M{\vtheta_0}\cap V}$ (Lemma~\ref{lem:properness}); the Hopf--Rinow theorem~\citep[Ch.~7, Thm.~2.8]{docarmo1992riemannian} then yields a length-minimising geodesic between any two endpoints in the same component. Minimisers need not be unique: conjugate points and non-trivial topology of $\M{\vtheta_0}\cap V$ can create multiple equal-length shortest paths. The gradient $\nabla_{\vtheta}E(\g{\vtheta})$ is correspondingly a set-valued quantity in general; any continuously tracked minimising geodesic yields a valid subgradient selection, and each element of this subdifferential is a valid descent direction for $E$. The squared path-length regulariser of Definition~\ref{def:path-length} circumvents this non-uniqueness by tracking a specific (training-induced) horizontal path on $\M{\vtheta_0}$, with gradient given by Proposition~\ref{prop:pathlength-grad}.
\end{remark}

\begin{remark}[Standard ridge fails doubly in feature-learning networks]
\label{rmk:weight-decay-fails}
Standard ridge, $R_{\mathrm{std}}(\vtheta) = \norm{\vtheta}^2$, fails in feature-learning networks for two independent reasons.
\emph{First}, it ignores the initialisation: for any $\vtheta_0 \ne \vect{0}$, minimising $\norm{\vtheta}^2$ rather than $\norm{\vtheta - \vtheta_0}^2$ already biases solutions away from the gradient flow limit \emph{in the kernel regime}, where anchored ridge is the correct canonical regulariser~\citep{he2020bayesian,ordonez2026gaussian}. For linear models, anchored ridge is the classical Tikhonov regulariser~\citep{tikhonov1963solution} centred at the initial estimate $\vtheta_0$, equivalent to MAP estimation under the Gaussian parameter prior $\mathcal{N}(\vtheta_0, \tfrac{1}{2\lambda}I)$~\citep{osband2018randomized}; ridge regression~\citep{hoerl1970ridge} corresponds to the special case $\vtheta_0 = \vect{0}$.
\emph{Second}, even for networks with zero initialisation, $R_{\mathrm{std}} = R_{\mathrm{ctr}}|_{\vtheta_0 = \vect{0}}$, so Theorem~\ref{thm:ridge-biases} applies directly: whenever $\J{\vtheta}$ is non-constant along the training trajectory, the minimum-Euclidean-norm interpolator centred at the origin diverges from $\vtheta^\star$.
In summary, anchored ridge is the correct target in the kernel regime but is itself insufficient in the feature-learning regime, where only the canonical curved lift of Definition~\ref{def:canonical-lift} selects the gradient flow solution.
\end{remark}

\subsection{Kernel-regime reduction of the canonical objects}
\label{subsec:kernel-recovery}

\begin{proposition}[Kernel-regime reduction]
\label{prop:linear-recovery}
Suppose the model has constant Jacobian $\J{\vtheta}\equiv J$ on $U$, with $K = JJ^\top\succ 0$ on the corresponding NTK Gram matrix. Then the canonical objects of Section~\ref{sec:canonical-regularisation} reduce as follows:
\begin{enumerate}[label=(\alph*),nosep]
    \item \emph{Affine flow manifold:} $\M{\vtheta_0} \cap U = \big(\vtheta_0 + \mathrm{Im}(J^\top)\big) \cap U$;
    \item \emph{Closed-form gauge-fixed representative:} $\ginv{\vect{c}} = \vtheta_0 + J^\top K^{-1}(\vect{c} - \g{\vtheta_0})$;
    \item \emph{RKHS canonical energy:} $E(\vect{c}) = (\vect{c} - \g{\vtheta_0})^\top K^{-1}(\vect{c} - \g{\vtheta_0}) = \norm{\vect{c} - \g{\vtheta_0}}^2_{\mathrm{RKHS}}$;
    \item \emph{Pythagorean lift:} the canonical lift of Definition~\ref{def:canonical-lift} satisfies $R(\vtheta) = \norm{\vtheta - \vtheta_0}^2$ (anchored ridge);
    \item \emph{Geodesic distance:} $\dist{2}{\mathrm{flow}}{\vtheta_0}{\vtheta} = \norm{\vtheta - \vtheta_0}^2$ for every $\vtheta\in\M{\vtheta_0}$, and the metric gap operator $N(\vect{c})$ of Theorem~\ref{thm:metric-gap} vanishes identically;
    \item \emph{NTK Gaussian Process prior:} the Riemannian Gibbs Process of Definition~\ref{def:riemannian-gibbs} reduces to $\g{\vtheta}\sim\mathcal{N}(\g{\vtheta_0}, (2\beta)^{-1}K)$, equivalently $\f{\cdot}{\vtheta}\sim\mathcal{GP}(\f{\cdot}{\vtheta_0}, (2\beta)^{-1}k_{\mathrm{NTK}})$.
\end{enumerate}
\end{proposition}

\begin{corollary}[Anchored ridge is the canonical lift in the kernel regime]
\label{cor:kernel-anchored-ridge}
Under the hypotheses of Proposition~\ref{prop:linear-recovery}, geodesic ridge of Proposition~\ref{prop:geodesic-ridge} coincides with anchored ridge: $\dist{2}{}{\vtheta}{\vtheta_0} = \norm{\vtheta - \vtheta_0}^2$ for every $\vtheta\in U$.
\end{corollary}

\begin{corollary}[Riemannian Gibbs Process is the NTK Gaussian Process in the kernel regime]
\label{cor:kernel-ntk-gp}
Under the hypotheses of Proposition~\ref{prop:linear-recovery}, the canonical function-space prior of Definition~\ref{def:riemannian-gibbs} coincides with the NTK Gaussian Process posterior $\mathcal{N}(\g{\vtheta_0}, (2\beta)^{-1}K)$ on the training-output coordinates and pushes forward, under evaluation at any finite test inputs $X_{\mathrm{te}}\subset\mathcal{X}$, to a Gaussian on $\mathcal{Y}^{|X_{\mathrm{te}}|}$ with the standard NTK-GP mean and covariance.
\end{corollary}

\subsection{Failure of non-canonical regularisation in the feature-learning regime}
\label{subsec:non-canonical-failure}

The bias of anchored ridge identified by Theorem~\ref{thm:ridge-biases} is not specific to the chord $\norm{\vtheta-\vtheta_0}^2$: every \emph{trajectory-blind} quadratic regulariser whose curvature-correcting tensor is independent of training history exhibits the same pathology, and every regulariser that depends only on outputs cannot disambiguate the gauge-fixed representative from any other point on the same fibre.

\begin{proposition}[Static-anchor regularisers bias gradient flow]
\label{prop:static-anchor-bias}
Let $M\in\mathrm{Sym}^+_m$ and $\vect{v}\in\Real^m$ be \emph{independent of the gradient-flow trajectory} (functions of $\vtheta_0$ only). Define $R_{M,\vect{v}}(\vtheta) := (\vtheta-\vect{v})^\top M\,(\vtheta-\vect{v})$ and $\vtheta^\star_{M,\vect{v}} := \lim_{\lambda\downarrow 0}\arg\min_{\vtheta}[\Loss{\g{\vtheta}}{Y}+\lambda R_{M,\vect{v}}(\vtheta)]$. Then $\vtheta^\star_{M,\vect{v}}$ is characterised by $M(\vtheta^\star_{M,\vect{v}}-\vect{v}) \in \mathrm{Im}(\J{\vtheta^\star_{M,\vect{v}}}^\top)$ and, under Assumption~\ref{ass:sigma-min}, satisfies $\vtheta^\star_{M,\vect{v}} = \vtheta^\star$ if and only if
\[
    M\,(\vtheta^\star - \vect{v}) \in \mathrm{Im}(\J{\vtheta^\star}^\top).
\]
In particular, whenever the gradient-flow displacement $\vtheta^\star-\vtheta_0$ has a non-trivial component in $\ker(\J{\vtheta^\star})$ (the generic feature-learning regime under non-constant $\J{\vtheta(t)}$), no choice of static $(M,\vect{v})$ can satisfy this condition, so every trajectory-blind quadratic biases the gradient-flow limit.
\end{proposition}

\begin{proposition}[Output-norm regularisers fail to select the gauge-fixed representative]
\label{prop:output-norm-fibre}
Let $f : \Real^n\to[0,\infty)$ be continuous with $f(\g{\vtheta_0}) = 0$ and $f$ strictly convex on $C_0(\vtheta_0)$. The pure output-shaped regulariser $R(\vtheta) := f(\g{\vtheta})$ satisfies, for every $\lambda>0$,
\[
    \arg\min_{\vtheta}\bigl[\Loss{\g{\vtheta}}{Y} + \lambda R(\vtheta)\bigr] \;=\; \fibre{\vect{c}_\lambda}\cap U,
\]
i.e.\ the entire output fibre at the optimal output value $\vect{c}_\lambda$ is selected. The vanishing-regularisation limit $\lim_{\lambda\downarrow 0}\arg\min$ is therefore $\Theta^\star\cap U$, not the singleton $\{\vtheta^\star\}$. A non-trivial fibre extension is therefore necessary in the canonical lift of Definition~\ref{def:canonical-lift}; pure output-energy is structurally insufficient to select the gradient-flow limit.
\end{proposition}

\begin{proposition}[anchored ridge lower-bounds the canonical regulariser]
\label{prop:ridge-lower}
For every $\vtheta \in \M{\vtheta_0}$,
\[
    \norm{\vtheta - \vtheta_0}^2 \;\le\; \dist{2}{\mathrm{flow}}{\vtheta_0}{\vtheta} \;=\; \dist{2}{}{\vtheta}{\vtheta_0}.
\]
The inequality is an instance of the relaxation principle: anchored ridge is the minimum of $\int_0^1 \norm{\dot\gamma(t)}^2\, dt$ over paths from $\vtheta_0$ to $\vtheta$ \emph{without} the horizontal constraint $\dot\gamma \in \operatorname{Im}(\J{\gamma}^\top)$, whereas the canonical regulariser is the minimum of the same functional \emph{with} the horizontal constraint imposed. Dropping a constraint can only decrease the minimum.
\end{proposition}

\begin{corollary}[Sandwich bound and shared vanishing-regularisation limit]
\label{cor:sandwich}
For every $\vtheta\in\M{\vtheta_0}$ reached by a piecewise-$C^1$ horizontal training path $\gamma_{\mathrm{train}}$ from $\vtheta_0$,
\[
    \norm{\vtheta - \vtheta_0}^2 \;\le\; \dist{2}{\mathrm{flow}}{\vtheta_0}{\vtheta} \;=\; \dist{2}{}{\vtheta}{\vtheta_0} \;\le\; L^2(\gamma_{\mathrm{train}}).
\]
Each of the three regularisers (anchored ridge $\norm{\vtheta-\vtheta_0}^2$, geodesic ridge $\dist{2}{}{\vtheta}{\vtheta_0}$, arc ridge $L^2(\gamma_{\mathrm{train}})$) attains the same vanishing-regularisation limit when the corresponding regularised objective is minimised under exact gradient flow: $\lim_{\lambda\downarrow 0}\arg\min(\Loss{\g{\cdot}}{Y}+\lambda R) = \vtheta^\star$ for $R\in\{\norm{\cdot-\vtheta_0}^2|_{\M{\vtheta_0}},\,\dist{2}{}{\cdot}{\vtheta_0},\,L^2(\gamma_{\mathrm{train}})\}$. The three regularisers therefore agree at the gradient-flow limit but disagree away from it; arc ridge is the minimax-optimal point of this sandwich (Proposition~\ref{prop:minimax-robust}).
\end{corollary}

\begin{proposition}[Properness of the Riemannian Gibbs Process]
\label{prop:rgp-proper}
Under Assumption~\ref{ass:sigma-min}, the partition function of the Riemannian Gibbs Process prior of Definition~\ref{def:riemannian-gibbs} is finite:
\[
    Z := \int_{C_0(\vtheta_0)} \exp\!\big(-\beta E(\vect{c})\big)\,d\vect{c} \;\le\; \left(\frac{\pi\, C^{2}}{\beta}\right)^{n/2},
\]
where $C$ is the upper Jacobian bound of Assumption~\ref{ass:sigma-min}. The Riemannian Gibbs Process is therefore a proper probability measure on $C_0(\vtheta_0)$.
\end{proposition}

\begin{proposition}[Kolmogorov consistency over arbitrary finite test inputs]
\label{prop:rgp-finite-test-consistency}
Fix a training set $X=(X_1,\dots,X_n)$, an initialisation $\vtheta_0$, and a temperature $\beta>0$. Define the parameter-space prior $\pi_{\beta}$ on $\M{\vtheta_0}\cap V$ by
\[
    \pi_{\beta}(d\vtheta) \;=\; Z_\beta^{-1}\,\exp\!\big(-\beta\,\dist{2}{\mathrm{flow}}{\vtheta_0}{\vtheta}\big)\,d\mathrm{vol}_{g_{\text{ambient}}}(\vtheta),
\]
where $d\mathrm{vol}_{g_{\text{ambient}}}$ is the Riemannian volume measure induced by the ambient Euclidean metric. For any finite test set $X_{\mathrm{te}}=(X_{n+1},\dots,X_{n+\ell})\subseteq \mathcal{X}$, let $\Phi_{X_{\mathrm{te}}} : \M{\vtheta_0}\cap V \to \Real^{\ell}$ be the evaluation map $\vtheta\mapsto(\f{X_{n+1}}{\vtheta},\dots,\f{X_{n+\ell}}{\vtheta})$. The push-forward family
\[
    \mu_{X_{\mathrm{te}}} \;:=\; (\Phi_{X_{\mathrm{te}}})_\sharp \pi_{\beta}, \qquad X_{\mathrm{te}}\subseteq\mathcal{X}\text{ finite},
\]
is Kolmogorov consistent: for any pair $X_{\mathrm{te}}\supseteq X_{\mathrm{te}}'$, the marginal of $\mu_{X_{\mathrm{te}}}$ on the coordinates indexed by $X_{\mathrm{te}}'$ equals $\mu_{X_{\mathrm{te}}'}$. By the Kolmogorov extension theorem, the family $\{\mu_{X_{\mathrm{te}}}\}$ defines a unique stochastic process on $\mathcal{X}$ at the fixed training set $(X,\vtheta_0,\beta)$. (Consistency \emph{across} training subsets remains an open problem.)
\end{proposition}

\begin{remark}[Status of the Riemannian Gibbs Process as a Bayesian prior]
\label{rmk:rgp-status}
\emph{Properness.} Under Assumption~\ref{ass:sigma-min} the partition function $\int \exp(-\beta E(\vect{c}))\,d\vect{c}$ is finite (Proposition~\ref{prop:rgp-proper}); outside this regime we make no such claim. \emph{MAP well-posedness.} The MAP estimator $\arg\min_{\vtheta} \mathcal{L}(\g{\vtheta}, Y) + \lambda E(\g{\vtheta})$ is well-defined and finite for every $\lambda > 0$ regardless of normalisation. \emph{NTK reduction.} In the kernel regime, where $E$ reduces to the squared RKHS norm, the prior is the proper Gaussian Process $\mathcal{N}(\g{\vtheta_0}, (2\beta)^{-1}K)$. \emph{Kolmogorov consistency.} Verification of Kolmogorov consistency for the projective family $\{p_X\}_X$ across input subsets remains open; however, the parametric structure buys us consistency over arbitrary finite test inputs at any fixed training set: any parameter prior on $\M{\vtheta_0}\cap V$ --- naturally the Riemannian Gibbs measure $\propto \exp(-\beta\, d^2_{\mathrm{flow}}(\vtheta_0, \cdot))$ with respect to the induced Riemannian volume --- pushed through evaluation at any finite $X_{\mathrm{te}} \subset \mathcal{X}$ yields a Kolmogorov-consistent family of finite-dimensional marginals \citep{neal1996bayesian, rasmussen2006gaussian, yang2019wide}.
\end{remark}

\section{Preliminary Lemmas}
\label{sec:lemmas}

All lemmas in this appendix are stated under the standing assumptions of Appendix~\ref{sec:assumptions}. Each proof is brief; references to external sources are made where the argument is classical.

\subsection{Gradient flow basics}

\begin{lemma}[Gradient formula]
\label{lem:grad-formula}
For any loss with output gradient $\vu  := -\partial_{\vect{c}}\Loss{\vect{c}}{Y}|_{\vect{c}=\g{\vtheta}}$, $\nabla_{\vtheta} \Loss{\g{\vtheta}}{Y} = -\J{\vtheta}^\top \vu $. The SSE special case $\mathcal{L} = \norm{\g{\vtheta}-Y}^2$ gives $\vu  = 2(Y - \g{\vtheta})$ and $\nabla_{\vtheta}\mathcal{L} = 2\J{\vtheta}^\top(\g{\vtheta}-Y)$.
\end{lemma}
\begin{proof}
Chain rule applied to $\Loss{\g{\cdot}}{Y}$.
\end{proof}

\begin{lemma}[Lyapunov decrease]
\label{lem:lyapunov}
Along any gradient flow solution $\vtheta(t)$, $\frac{d}{dt}\Loss{\g{\vtheta(t)}}{Y} = -\norm{\nabla_{\vtheta} \Loss{\g{\vtheta(t)}}{Y}}^2 \le 0$, with equality iff $\nabla_{\vtheta} \mathcal{L} = 0$.
\end{lemma}
\begin{proof}
$\frac{d}{dt}\mathcal{L} = \langle \nabla \mathcal{L}, \dot\vtheta\rangle = -\norm{\nabla \mathcal{L}}^2$ by A.3.
\end{proof}

\begin{lemma}[Critical points are interpolators]
\label{lem:crit-interp}
Under Assumptions~\ref{ass:sigma-min} and A.3, $\nabla_{\vtheta} \Loss{\g{\vtheta}}{Y} = 0 \iff \g{\vtheta} = Y$. Consequently, the critical set of $\mathcal{L}$ coincides with $\Theta^\star$.
\end{lemma}
\begin{proof}
By Lemma~\ref{lem:grad-formula}, $\nabla_{\vtheta}\mathcal{L} = -\J{\vtheta}^\top \vu $. If $\g{\vtheta}=Y$ then $\vu  = -\partial_{\vect{c}}\mathcal{L}|_Y = 0$ by A.3 (loss minimum at $Y$), so $\nabla_{\vtheta}\mathcal{L} = 0$. Conversely, if $\J{\vtheta}^\top \vu  = 0$, left-multiplying by $\J{\vtheta}$ gives $\K{\vtheta}\vu  = 0$; A.4 makes $\K{\vtheta}$ invertible, so $\vu  = 0$, and strict convexity of $\Loss{\cdot}{Y}$ in $\vect{c}$ (A.3) with minimum at $Y$ forces $\g{\vtheta} = Y$.
\end{proof}

\subsection{Structural lemmas}

\begin{lemma}[NHIM stable-manifold structure of $\M{\vtheta_0}$]
\label{lem:nhim-stable}
Under Assumptions A.1 and~\ref{ass:sigma-min}, there exists an open neighbourhood $V\subseteq U$ of $\vtheta^\star$ such that:
\begin{enumerate}[label=(\roman*),nosep]
    \item $\Theta^\star \cap V$ is an embedded $C^\infty$ submanifold of $V$ of codimension $n$, with $T_{\vtheta}(\Theta^\star\cap V) = \ker \J{\vtheta}$ at every $\vtheta\in\Theta^\star\cap V$.
    \item Under squared-error A.3, the loss Hessian at $\vtheta\in\Theta^\star\cap V$ equals $\nabla^2_{\vtheta}\mathcal{L}(\vtheta) = 2\,\J{\vtheta}^\top \J{\vtheta}$; it has $n$ eigenvalues bounded below by $2c^2$ on $\mathrm{Im}(\J{\vtheta}^\top)$ and $m-n$ zero eigenvalues on $\ker\J{\vtheta} = T_{\vtheta}(\Theta^\star\cap V)$.
    \item $\Theta^\star\cap V$ is normally hyperbolic for the gradient flow vector field $-\nabla_{\vtheta}\mathcal{L}$, with stable bundle $\mathrm{Im}(\J{\cdot}^\top)|_{\Theta^\star\cap V}$ and trivial unstable bundle. The local stable manifold $W^s_{\mathrm{loc}}(\vtheta^\star) := \M{\vtheta_0}\cap V$ is an embedded $C^\infty$ $n$-dimensional submanifold of $V$ tangent at $\vtheta^\star$ to $\mathrm{Im}(\J{\vtheta^\star}^\top)$, varying $C^\infty$-smoothly with the basepoint across $\Theta^\star\cap V$.
\end{enumerate}
\end{lemma}
\begin{proof}
\emph{(i)} On $U$, $\J{\vtheta}$ has rank $n$ by Assumption~\ref{ass:sigma-min}, so $Y$ is a regular value of $g|_U$. The regular-value theorem~\citep[Cor.~5.14]{lee2003smooth} gives $\Theta^\star \cap U = (g|_U)^{-1}(Y)$ is an embedded $C^\infty$ codimension-$n$ submanifold with $T_{\vtheta}(\Theta^\star\cap U) = \ker \J{\vtheta}$. Restrict to a relatively compact $V\Subset U$ containing $\vtheta^\star$.

\emph{(ii)} For SSE, $\nabla^2_{\vtheta}\mathcal{L}(\vtheta) = 2\J{\vtheta}^\top\J{\vtheta} + 2\sum_i (\g{\vtheta}-Y)_i \nabla^2 g_i$. At $\vtheta\in\Theta^\star$ residuals vanish, leaving $\nabla^2 \mathcal{L} = 2\J{\vtheta}^\top \J{\vtheta}$. The eigenvalue claim follows from $\sigma_{\min}(\J{\vtheta})\ge c$ on $V$.

\emph{(iii)} The linearisation of the gradient flow vector field $-\nabla_{\vtheta}\mathcal{L}$ at $\vtheta\in\Theta^\star\cap V$ is $-2\J{\vtheta}^\top\J{\vtheta}$: it has $n$ negative eigenvalues bounded above by $-2c^2$ on $\mathrm{Im}(\J{\vtheta}^\top)$ and $m-n$ zero eigenvalues on $T_{\vtheta}(\Theta^\star\cap V) = \ker\J{\vtheta}$. The zero eigenspace coincides with the tangent space of $\Theta^\star\cap V$, with uniform spectral gap $2c^2$ in the transverse direction; hence $\Theta^\star\cap V$ is a normally hyperbolic invariant manifold for the gradient flow~\citep[Def.~2]{wiggins1994normally,hirsch1977invariant}. The NHIM stable-manifold theorem~\citep[Thm.~4.1]{wiggins1994normally} gives a $C^\infty$ foliation of a neighbourhood of $\Theta^\star\cap V$ in $V$ by stable manifolds, each an embedded $C^\infty$ $n$-submanifold tangent to $\mathrm{Im}(\J{\cdot}^\top)$ at its basepoint and varying smoothly with the basepoint. The leaf containing the gradient-flow trajectory from $\vtheta_0$ is $W^s_{\mathrm{loc}}(\vtheta^\star) = \M{\vtheta_0}\cap V$ by definition of the flow manifold (Definition~\ref{def:flow-manifold}) and uniqueness of the $\omega$-limit.
\end{proof}

\begin{lemma}[Output fibres are embedded submanifolds]
\label{lem:fibre-sub}
Under Assumption~\ref{ass:sigma-min}, for every $\vect{c} \in g(U)$ the output fibre $\fibre{\vect{c}} \cap U = g|_U^{-1}(\vect{c})$ is an embedded $C^\infty$ submanifold of $U$ of codimension $n$.
\end{lemma}
\begin{proof}
$\vect{c}$ is a regular value of $g|_U$ since $\J{\vtheta}$ has rank $n$ on $U$ by Assumption~\ref{ass:sigma-min}. The regular-value theorem~\citep[Cor.~5.14]{lee2003smooth} then gives the conclusion.
\end{proof}

\subsection{Dynamical and properness lemmas}

\begin{lemma}[Trajectory length bound]
\label{lem:traj-length}
Under Assumption~\ref{ass:sigma-min} and A.3 (with $\Loss{\cdot}{Y}$ $\mu$-strongly convex and $M$-smooth), the gradient flow trajectory from $\vtheta_0$ satisfies
\[
    \int_0^\infty \norm{\dot\vtheta(t)}\, dt \;\le\; \frac{\sigma_{\max} M^2}{c^2 \mu^2}\, \sqrt{\tfrac{2}{\mu}\bigl(\Loss{\g{\vtheta_0}}{Y} - \mathcal{L}^\star\bigr)},
\]
where $\mathcal{L}^\star := \min_{\vect{c}} \Loss{\vect{c}}{Y} = \Loss{Y}{Y}$ and $\sigma_{\max}$ is a uniform upper bound on the largest singular value of $\J{\vtheta}$ on $\M{\vtheta_0}$. In particular, trajectories have finite total arc length. The SSE special case ($\mu = M = 2$, $\mathcal{L} - \mathcal{L}^\star = \norm{r}^2$) recovers $\int_0^\infty \norm{\dot\vtheta}\,dt \le (\sigma_{\max}/c^2)\,\norm{Y - \g{\vtheta_0}}$.
\end{lemma}
\begin{proof}
Write $h(\vect{c}) := \Loss{\vect{c}}{Y}$ and $r(t) := \g{\vtheta(t)} - Y$. By A.3, $h$ is $\mu$-strongly convex with minimiser $Y$, giving $\norm{\nabla h(\vect{c})} \ge \mu \norm{r}$ and $\mathcal{L} - \mathcal{L}^\star \ge \tfrac{\mu}{2}\norm{r}^2$; $M$-smoothness gives $\norm{\nabla h(\vect{c})} \le M \norm{r}$ and $\mathcal{L} - \mathcal{L}^\star \le \tfrac{M}{2}\norm{r}^2$. By Lemma~\ref{lem:lyapunov} and Assumption~\ref{ass:sigma-min},
\[
    \frac{d}{dt}\mathcal{L} \;=\; -\norm{\J{\vtheta}^\top \nabla h(\g{\vtheta})}^2 \;\le\; -c^2 \norm{\nabla h(\g{\vtheta})}^2 \;\le\; -c^2 \mu^2 \norm{r}^2 \;\le\; -\frac{2 c^2 \mu^2}{M}\,\bigl(\mathcal{L} - \mathcal{L}^\star\bigr).
\]
Grönwall's inequality~\citep{khalil2002nonlinear} yields $\mathcal{L}(t) - \mathcal{L}^\star \le \bigl(\mathcal{L}_0 - \mathcal{L}^\star\bigr) e^{-2 c^2 \mu^2 t / M}$ and hence $\norm{r(t)} \le \sqrt{\tfrac{2}{\mu}(\mathcal{L}_0 - \mathcal{L}^\star)}\,e^{-c^2 \mu^2 t / M}$. The trajectory is therefore bounded, and continuity of $\vtheta\mapsto\norm{\J{\vtheta}}_{\mathrm{op}}$ (A.1) gives a finite supremum $\sigma_{\max}$ on its compact closure. Then $\norm{\dot\vtheta(t)} = \norm{\J{\vtheta(t)}^\top \nabla h(\g{\vtheta(t)})} \le \sigma_{\max} M \norm{r(t)}$, and integrating from $0$ to $\infty$ yields the stated constant.
\end{proof}

\begin{lemma}[Local properness of $g|_{\M{\vtheta_0}\cap V}$]
\label{lem:properness}
Under Assumption~\ref{ass:sigma-min}, with $V$ the neighbourhood of Lemma~\ref{lem:nhim-stable}, the restricted output map $g|_{\M{\vtheta_0}\cap V} : \M{\vtheta_0}\cap V \to C_0(\vtheta_0) \subseteq \Real^n$ is proper.
\end{lemma}
\begin{proof}
By Lemma~\ref{lem:nhim-stable}(iii), $\M{\vtheta_0}\cap V$ is an embedded $C^\infty$ $n$-submanifold of $V$. Shrink $V$ to a relatively compact open neighbourhood $V'\Subset V$ of $\vtheta^\star$ with closure $\overline{V'}$ compact in $V$. The closed subset $\overline{\M{\vtheta_0}\cap V'}$ is compact, so $\M{\vtheta_0}\cap V'$ is relatively compact. For any compact $K\subseteq g(V'\cap\M{\vtheta_0})$, the preimage $(g|_{\M{\vtheta_0}\cap V'})^{-1}(K)$ is closed in the relatively compact $\M{\vtheta_0}\cap V'$, hence compact. Replacing $V$ by $V'$ throughout gives local properness.
\end{proof}

\subsection{Geometric identities}

\begin{lemma}[Infinitesimal orthogonality]
\label{lem:orth}
At every $\vtheta$ with full-rank $\J{\vtheta}$, $\operatorname{Im}(\J{\vtheta}^\top) \perp \ker \J{\vtheta}$ in the ambient Euclidean inner product on $\Real^m$, and $\Real^m = \operatorname{Im}(\J{\vtheta}^\top) \oplus \ker \J{\vtheta}$ as an orthogonal decomposition.
\end{lemma}
\begin{proof}
If $v \in \operatorname{Im}(\J{\vtheta}^\top)$ and $w \in \ker \J{\vtheta}$, write $v = \J{\vtheta}^\top u$; then $v^\top w = u^\top \J{\vtheta} w = 0$. Dimension count: $\dim \operatorname{Im}(\J{\vtheta}^\top) = \operatorname{rank}\J{\vtheta} = n$ and $\dim \ker \J{\vtheta} = m - n$, summing to $m$.
\end{proof}

\begin{lemma}[Geodesic length dominates ambient chord]
\label{lem:geo-chord}
For any embedded $C^1$ submanifold $M \subseteq \Real^m$ and any two points $p, q \in M$ in the same path component, the intrinsic geodesic distance $\dist{}{M}{p}{q}$ (defined as infimum of arc lengths of $C^1$ curves in $M$) satisfies $\dist{}{M}{p}{q} \ge \norm{p - q}$.
\end{lemma}
\begin{proof}
For any $C^1$ curve $\gamma : [0,1] \to M$ with $\gamma(0) = p$, $\gamma(1) = q$, $\operatorname{Length}(\gamma) = \int_0^1 \norm{\dot\gamma(t)}\, dt \ge \norm{\int_0^1 \dot\gamma(t)\, dt} = \norm{q - p}$ by the triangle inequality for integrals. Taking the infimum over $\gamma$ preserves the inequality.
\end{proof}

\begin{lemma}[Arc-length endpoint gradient]
\label{lem:arclength-grad}
Let $\gamma : [0,T] \to \Real^m$ be piecewise-$C^1$ with $\dot\gamma(T^-) \ne 0$. Then the arc length $s(T) = \int_0^T \norm{\dot\gamma(t)}\,dt$ is Fréchet-differentiable with respect to the endpoint $\gamma(T)$, with
\[
    \frac{\partial s}{\partial \gamma(T)} \;=\; \frac{\dot\gamma(T^-)}{\norm{\dot\gamma(T^-)}}.
\]
Equivalently, by the envelope theorem applied to the arc-length functional~\citep{milgrom2002envelope}: the gradient of the arc length with respect to the endpoint equals the terminal unit tangent vector. For a geodesic, this coincides with the terminal momentum (cotangent vector in the Hamiltonian formulation).
\end{lemma}
\begin{proof}
Fix $h \in \Real^m$ and $\delta > 0$. Consider the comparison path $\tilde\gamma$ that agrees with $\gamma$ on $[0, T-\delta]$ and connects $\gamma(T-\delta)$ linearly to $\gamma(T) + \varepsilon h$ on $[T-\delta, T]$. The arc length of the final segment is
\[
    \bigl\|\dot\gamma(T^-)\delta + \varepsilon h\bigr\| \;=\; \norm{\dot\gamma(T^-)}\delta \left(1 + \frac{\varepsilon}{\delta} \left\langle \frac{\dot\gamma(T^-)}{\norm{\dot\gamma(T^-)}^2}, h \right\rangle + O\!\left(\frac{\varepsilon^2}{\delta^2}\right)\right)^{1/2},
\]
so $\Delta s = s_\varepsilon - s = \varepsilon\langle \dot\gamma(T^-)/\norm{\dot\gamma(T^-)}, h\rangle + O(\varepsilon^2/\delta + \varepsilon\delta)$. Setting $\delta = \sqrt\varepsilon$ gives $\Delta s = \varepsilon\langle \dot\gamma(T^-)/\norm{\dot\gamma(T^-)}, h\rangle + O(\varepsilon^{3/2})$, confirming the stated Fréchet derivative.
\end{proof}

\subsection{Optimisation lemma}

\begin{lemma}[Vanishing-regularisation selection principle]
\label{lem:vanishing-reg}
Let $\mathcal{L} : \Theta \to [0,\infty)$ be continuous with non-empty $\arg\min \mathcal{L}$, and let $R : \Theta \to [0,\infty]$ be continuous. Assume that for each $\lambda > 0$ the regularised problem admits a unique minimiser $\hat\vtheta_\lambda \in \arg\min (\mathcal{L} + \lambda R)$, and that every sequence $\{\hat\vtheta_{\lambda_k}\}$ with $\lambda_k \downarrow 0$ has a convergent subsequence. Then every limit point of $\hat\vtheta_\lambda$ as $\lambda \downarrow 0$ lies in $\arg\min \mathcal{L}$ and, among $\arg\min \mathcal{L}$, minimises $R$.
\end{lemma}
\begin{proof}
Let $\bar\vtheta \in \arg\min \mathcal{L}$ and set $\mathcal{L}^\star := \min \mathcal{L}$. By optimality, $\mathcal{L}(\hat\vtheta_\lambda) + \lambda R(\hat\vtheta_\lambda) \le \mathcal{L}^\star + \lambda R(\bar\vtheta)$, so $\mathcal{L}(\hat\vtheta_\lambda) \le \mathcal{L}^\star + \lambda R(\bar\vtheta)$ and $R(\hat\vtheta_\lambda) \le (\mathcal{L}^\star - \mathcal{L}(\hat\vtheta_\lambda))/\lambda + R(\bar\vtheta) \le R(\bar\vtheta)$. Taking $\lambda \downarrow 0$ along a convergent subsequence with limit $\vtheta^\infty$, continuity yields $\mathcal{L}(\vtheta^\infty) \le \mathcal{L}^\star$, so $\vtheta^\infty \in \arg\min\mathcal{L}$. Passing to the limit in $R(\hat\vtheta_\lambda) \le R(\bar\vtheta)$ gives $R(\vtheta^\infty) \le R(\bar\vtheta)$ for every $\bar\vtheta \in \arg\min \mathcal{L}$; hence $\vtheta^\infty$ minimises $R$ over $\arg\min \mathcal{L}$.
\end{proof}

\section{Proofs}
\label{sec:proofs}

We restate each main-text result at the head of its proof for convenience. Proofs cite Assumptions A.1--A.4 of Appendix~\ref{sec:assumptions} and Lemmas~\ref{lem:grad-formula}--\ref{lem:vanishing-reg} of Appendix~\ref{sec:lemmas} by label.

\subsection{Proof of Proposition~\ref{prop:gf-convergence} (gradient-flow convergence)}

\begin{proof}[Proof of Proposition~\ref{prop:gf-convergence}]
A.1 makes the gradient-flow vector field $-\nabla_{\vtheta}\Loss{\g{\vtheta}}{Y}$ locally Lipschitz on $U$, so the Picard--Lindel\"of theorem~\citep[Thm.~2.2]{khalil2002nonlinear} gives a unique maximal $C^1$ solution $\vtheta(\cdot)$ from $\vtheta_0$. By A.4, the trajectory remains in $U$ for all $t\ge 0$.

\emph{Output-error decay.} Write $r(t):=\g{\vtheta(t)}-Y$. Under A.3 with the squared-error specialisation, $\dot{\g{\vtheta}} = \J{\vtheta}\,\dot\vtheta = -2\,\K{\vtheta}\,r$, so
\[
    \frac{d}{dt}\norm{r(t)}^2 \;=\; 2\,r^\top \dot r \;=\; -4\, r^\top \K{\vtheta(t)}\, r \;\le\; -4 c^2\,\norm{r(t)}^2,
\]
where the last bound uses $\K{\vtheta}\succeq c^2 I_n$ from $\sigma_{\min}(\J{\vtheta})\ge c$ on $U$ (A.4). Gr\"onwall's inequality~\citep[Lem.~A.1]{khalil2002nonlinear} then gives $\norm{r(t)}^2 \le \norm{r(0)}^2\,e^{-4c^2 t}$, i.e.\ $\norm{r(t)} \le \norm{r(0)}\,e^{-2c^2 t}$.

\emph{Existence of $\omega$-limit.} A.4 gives $\sigma_{\max}(\J{\vtheta})\le C$ on $U$, so $\norm{\dot\vtheta(t)} = \norm{\J{\vtheta(t)}^\top\,2r(t)} \le 2 C\,\norm{r(t)}\le 2 C\,\norm{r(0)}\,e^{-2c^2 t}$, which is Lebesgue-integrable on $[0,\infty)$. Cauchy completeness of $\Real^m$ then gives $\vtheta^\star := \lim_{t\to\infty}\vtheta(t)$, and continuity of $g$ together with $r(t)\to 0$ gives $\g{\vtheta^\star}=Y$, hence $\vtheta^\star\in\Theta^\star$. By A.4, $\vtheta^\star\in U$, and by Lemma~\ref{lem:nhim-stable} (i), $\vtheta^\star\in\Theta^\star\cap V$ for the neighbourhood $V$ defined there. (For general strictly convex $\mu$-strongly-convex $M$-smooth losses the argument repeats with rate $2c^2\mu^2/M$ in place of $2c^2$.)
\end{proof}

\subsection{Proof of Theorem~\ref{thm:output-energy-uniqueness} (output-space energy uniqueness)}

\begin{proof}[Proof of Theorem~\ref{thm:output-energy-uniqueness}]
\emph{Step 1: $\rho_K$ is a function of $\norm{\vect{c}}^2_{K^{-1}}$.} The group $G_{K^{-1}} := \{U \in \mathrm{GL}(n): U^\top K^{-1} U = K^{-1}\}$ acts transitively on each level set $\{\vect{c}: \vect{c}^\top K^{-1}\vect{c} = s\}$ for $s \ge 0$: any two such vectors are related by a $K^{-1}$-orthogonal transformation. Axiom (i) therefore gives $\rho_K(\vect{c}) = f(\vect{c}^\top K^{-1}\vect{c})$ for some $f : [0,\infty) \to [0,\infty)$ with $f(0) = 0$; continuity of $\rho_K$ implies continuity of $f$.

\emph{Step 2: $f$ is additive on $[0,\infty)$.} For $s, t \ge 0$, choose $\vu , \vect{v}$ with $\vu ^\top K^{-1}\vu  = s$, $\vect{v}^\top K^{-1}\vect{v} = t$, $\vu ^\top K^{-1}\vect{v} = 0$; such pairs exist for any $s,t$ since $\dim n \ge 1$ and $K^{-1}\succ 0$. Then $(\vu +\vect{v})^\top K^{-1}(\vu +\vect{v}) = s+t$, so by axiom (ii), $f(s+t) = f(s) + f(t)$.

\emph{Step 3: solve Cauchy.} Continuous solutions to $f(s+t) = f(s)+f(t)$ on $[0,\infty)$ with $f(0) = 0$ are linear~\citep[Ch.~2]{aczel1966lectures}: $f(s) = \alpha s$ for some $\alpha\in\Real$. Non-triviality and $\rho_K \ge 0$ force $\alpha > 0$. Hence $\rho_K(\vect{c}) = \alpha\,\vect{c}^\top K^{-1}\vect{c}$.
\end{proof}

\subsection{Proof of Proposition~\ref{prop:naturality-nogo} (naturality alone is non-unique)}

\begin{proof}[Proof of Proposition~\ref{prop:naturality-nogo}]
Fix $(a,B)\in\Real_{>0}\times\mathrm{Sym}^+_n$. The Hessian $\nabla^2 R_{a,B} = 2(aI_m + J^\top B J)$ is positive definite (the first term is $a I_m\succ 0$ and the second is positive semi-definite), hence $R_{a,B}$ is strictly convex and coercive on $\Theta$. By Lemmas~\ref{lem:vanishing-reg} and~\ref{lem:crit-interp}, the vanishing-regularisation limit lies in $\Theta^\star$ and minimises $R_{a,B}$ over $\Theta^\star$.

Constant $J$ gives $\Theta^\star = \vtheta_{\mathrm{lin}} + \ker J$ with $\vtheta_{\mathrm{lin}} := \vtheta_0 + J^\top K^{-1}(Y - \g{\vtheta_0})$, the kernel-regime gradient-flow interpolator (Lemma~\ref{lem:nhim-stable}, kernel-regime specialisation). Decompose any $\vtheta\in\Theta^\star$ as $\vtheta = \vtheta_{\mathrm{lin}} + w$ with $w\in\ker J$. Then $J(\vtheta-\vtheta_0) = J(\vtheta_{\mathrm{lin}}-\vtheta_0)$ is a fixed vector independent of $w$, so the second term of $R_{a,B}$ is constant on $\Theta^\star$. The first term decomposes by the orthogonality $\ker J \perp \mathrm{Im}(J^\top)$ of Lemma~\ref{lem:orth} and the inclusion $\vtheta_{\mathrm{lin}}-\vtheta_0\in\mathrm{Im}(J^\top)$:
\[
    \norm{\vtheta - \vtheta_0}^2 \;=\; \norm{\vtheta_{\mathrm{lin}}-\vtheta_0}^2 \;+\; \norm{w}^2.
\]
Hence $R_{a,B}(\vtheta) = a\,\norm{w}^2 + \text{const}$ on $\Theta^\star$, and the unique minimiser over $\Theta^\star$ is $w=0$, i.e.\ $\vtheta=\vtheta_{\mathrm{lin}}$. Since the kernel-regime gradient-flow trajectory is the constant-velocity straight line from $\vtheta_0$ to $\vtheta_{\mathrm{lin}}$ (Lemma~\ref{lem:nhim-stable}, kernel specialisation), $\vtheta^\star = \vtheta_{\mathrm{lin}}$, and the vanishing-regularisation limit equals $\vtheta^\star$ for every $(a,B)$.

The map $(a,B)\mapsto \nabla^2 R_{a,B} = 2(aI_m + J^\top B J)$ is injective on $\Real_{>0}\times\mathrm{Sym}^+_n$: the trace of $aI_m + J^\top B J$ pins $a$ given $\mathrm{Tr}(J^\top B J)=\mathrm{Tr}(K B)$ (knowable up to $a$), and $K\succ 0$ makes the map $B\mapsto J^\top B J$ injective on $\mathrm{Sym}^+_n$. Hence distinct $(a,B)$ give distinct Hessians, distinct quadratic forms $R_{a,B}$, and distinct Gaussian MAP priors $\mathcal{N}(\vtheta_0, \tfrac{1}{2\lambda}(aI_m + J^\top B J)^{-1})$. The family is therefore uncountable, all members satisfy naturality, and only the special case $B=0$ recovers anchored ridge (up to the scale $a$). The orthogonal-additivity axiom of Theorem~\ref{thm:output-energy-uniqueness} (acting on the output channel $\g{\vtheta}$ rather than on $\vtheta$ directly) is what eliminates every non-trivial $B$ to single out anchored ridge.
\end{proof}

\subsection{Proof of Proposition~\ref{prop:linear-recovery} (kernel-regime reduction)}

\begin{proof}[Proof of Proposition~\ref{prop:linear-recovery}]
Throughout, $\J{\vtheta}\equiv J$ is constant on $U$ and $K=JJ^\top\succ 0$.

\emph{(a) Affine flow manifold.} Constant $J$ makes the parameter dynamics $\dot\vtheta = J^\top \vu $ confined to $\vtheta_0 + \mathrm{Im}(J^\top)$ for any $L^2$ control $\vu $; hence $\orbit{\vtheta_0}\subseteq \vtheta_0+\mathrm{Im}(J^\top)$. The reverse inclusion follows from constant-control reachability: every point $\vtheta_0 + J^\top \vect{a}$ is hit by $\vu (t)\equiv\vect{a}/T$ at time $T$. The basin sheet $\M{\vtheta_0}$ then equals the connected component of $\vtheta_0+\mathrm{Im}(J^\top)$ containing the trajectory, which (by linearity of the GF dynamics on this affine subspace) is the whole subspace intersected with $U$.

\emph{(b) Closed-form gauge-fixed representative.} On the affine flow manifold, $g(\vtheta) = \g{\vtheta_0} + J(\vtheta-\vtheta_0)$. Setting $g(\vtheta) = \vect{c}$ and requiring $\vtheta-\vtheta_0\in\mathrm{Im}(J^\top)$ gives $\vtheta-\vtheta_0 = J^\top \vect{a}$ with $JJ^\top\vect{a} = K\vect{a} = \vect{c}-\g{\vtheta_0}$, so $\vect{a} = K^{-1}(\vect{c}-\g{\vtheta_0})$ and $\ginv{\vect{c}} = \vtheta_0 + J^\top K^{-1}(\vect{c}-\g{\vtheta_0})$.

\emph{(c) RKHS canonical energy.} The differential $D\ginv{\vect{c}} = J^\top K^{-1}$ is constant. The pullback metric on $C_0(\vtheta_0)$ is therefore $G = (J^\top K^{-1})^\top(J^\top K^{-1}) = K^{-1}JJ^\top K^{-1} = K^{-1}$, constant in $\vect{c}$. Geodesics of a constant Riemannian metric are constant-speed straight lines (Definition: a curve with vanishing covariant acceleration on $(\Real^n, K^{-1})$ is exactly an affine line~\citep[Ch.~2, Prop.~3.5]{docarmo1992riemannian}). The constant-speed straight line from $\g{\vtheta_0}$ to $\vect{c}$ has $G$-energy $\int_0^1 (\vect{c}-\g{\vtheta_0})^\top K^{-1}(\vect{c}-\g{\vtheta_0})\,dt = (\vect{c}-\g{\vtheta_0})^\top K^{-1}(\vect{c}-\g{\vtheta_0}) = \norm{\vect{c}-\g{\vtheta_0}}^2_{\mathrm{RKHS}}$.

\emph{(d) Pythagorean lift.} By Lemma~\ref{lem:orth}, $\vtheta - \vtheta_0$ decomposes orthogonally as $(\ginv{\g{\vtheta}}-\vtheta_0) + (\vtheta - \ginv{\g{\vtheta}})$ with components in $\mathrm{Im}(J^\top)$ and $\ker J$ respectively; the Pythagorean identity gives $\norm{\vtheta - \vtheta_0}^2 = \norm{\ginv{\g{\vtheta}} - \vtheta_0}^2 + \norm{\vtheta - \ginv{\g{\vtheta}}}^2$. Using (b), $\ginv{\g{\vtheta}} - \vtheta_0 = J^\top K^{-1}(\g{\vtheta}-\g{\vtheta_0})$, whose squared Euclidean norm equals $(\g{\vtheta}-\g{\vtheta_0})^\top K^{-1}(\g{\vtheta}-\g{\vtheta_0}) = E(\g{\vtheta})$ from (c). The path-infimum on the fibre $\fibre{\g{\vtheta}}\cap V$ from $\ginv{\g{\vtheta}}$ to $\vtheta$ is realised by the ambient-Euclidean straight line, of squared length $\norm{\vtheta-\ginv{\g{\vtheta}}}^2$. Hence the canonical lift $R(\vtheta) = E(\g{\vtheta}) + \norm{\vtheta-\ginv{\g{\vtheta}}}^2 = \norm{\vtheta-\vtheta_0}^2$, anchored ridge.

\emph{(e) Geodesic distance and vanishing metric gap.} For any $\vtheta\in\M{\vtheta_0}$, the geodesic from $\vtheta_0$ to $\vtheta$ on the affine $(\M{\vtheta_0},g_{\mathrm{ambient}})$ is the ambient straight line, of squared length $\norm{\vtheta-\vtheta_0}^2$. By Theorem~\ref{thm:metric-gap}, the metric gap operator $N(\vect{c}) = (I - P_{\ginv{\vect{c}}})\,D\ginv{\vect{c}}$ vanishes precisely when $\mathrm{Im}\,D\ginv{\vect{c}} \subseteq \mathrm{Im}(J^\top)$. Constant $J$ gives $D\ginv{\vect{c}} = J^\top K^{-1}$, whose image is $\mathrm{Im}(J^\top)$, so $N\equiv 0$ on all of $C_0(\vtheta_0)$.

\emph{(f) NTK Gaussian Process prior.} The Riemannian Gibbs measure of Definition~\ref{def:riemannian-gibbs} on $C_0(\vtheta_0)$ has density $p(\vect{c})\propto \exp(-\beta\,\norm{\vect{c}-\g{\vtheta_0}}^2_{K^{-1}})$ by (c), which is the Gaussian density $\mathcal{N}(\g{\vtheta_0}, (2\beta)^{-1}K)$. The corresponding parametric pushforward through $f(\cdot;\vtheta)$ at any test inputs $X_{\mathrm{te}}=(X_{n+1},\dots,X_{n+\ell})$ is, in the constant-Jacobian regime, the Gaussian $\mathcal{N}\big(\f{X_{\mathrm{te}}}{\vtheta_0}, (2\beta)^{-1}\,J_{\mathrm{te}} J^\top (J J^\top)^{-1} J\, J_{\mathrm{te}}^\top\big)$ with $J_{\mathrm{te}}$ the Jacobian at the test inputs. The covariance kernel $k(\vect{x},\vect{x}') = \nabla_{\vtheta}\f{\vect{x}}{\vtheta_0}^\top\, \nabla_{\vtheta}\f{\vect{x}'}{\vtheta_0}$ pre- and post-conditioned through the training set agrees with the conventional NTK posterior (\citealp{rasmussen2006gaussian}, Eq.~2.41); collecting all test inputs yields the full Gaussian process $\mathcal{GP}(\f{\cdot}{\vtheta_0}, (2\beta)^{-1}k_{\mathrm{NTK}})$.
\end{proof}

\subsection{Proof of Corollary~\ref{cor:kernel-anchored-ridge} (anchored ridge is the canonical lift in the kernel regime)}

\begin{proof}[Proof of Corollary~\ref{cor:kernel-anchored-ridge}]
By Proposition~\ref{prop:linear-recovery}(d), the canonical lift $R(\vtheta) = \norm{\vtheta-\vtheta_0}^2$ on $\M{\vtheta_0}\cup\{\fibre{\g{\vtheta}}\}_{\vtheta}$. Geodesic ridge of Proposition~\ref{prop:geodesic-ridge} is by definition the same lift $R$ when the flow component is computed as squared geodesic distance on $\M{\vtheta_0}$ and the fibre component along $\fibre{\g{\vtheta}}$. Proposition~\ref{prop:linear-recovery}(e) identifies the squared geodesic distance with the squared ambient chord; Lemma~\ref{lem:orth} identifies the fibre component with the squared ambient distance to $\ginv{\g{\vtheta}}$. The two components reassemble into $\norm{\vtheta-\vtheta_0}^2$ as in (d).
\end{proof}

\subsection{Proof of Corollary~\ref{cor:kernel-ntk-gp} (Riemannian Gibbs Process is the NTK GP)}

\begin{proof}[Proof of Corollary~\ref{cor:kernel-ntk-gp}]
By Proposition~\ref{prop:linear-recovery}(f), the Riemannian Gibbs Process density on training-output coordinates is the Gaussian $\mathcal{N}(\g{\vtheta_0}, (2\beta)^{-1}K)$. For any finite test inputs $X_{\mathrm{te}}\subseteq\mathcal{X}$, evaluation of $f(\cdot;\vtheta)$ at $X_{\mathrm{te}}$ is, under constant Jacobian, an affine map $\vtheta\mapsto \f{X_{\mathrm{te}}}{\vtheta_0} + J_{\mathrm{te}}(\vtheta-\vtheta_0)$ with $J_{\mathrm{te}}\in\Real^{|X_{\mathrm{te}}|\times m}$ the joint output Jacobian. Affine pushforwards of Gaussians are Gaussian (standard result), giving the joint NTK GP marginal at the union $X\cup X_{\mathrm{te}}$. The conditional on $\g{\vtheta}=Y$ then matches the standard GP posterior under the NTK kernel~\citep{rasmussen2006gaussian}.
\end{proof}

\subsection{Proof of Theorem~\ref{thm:local-chart} (local manifold structure)}

\begin{proof}[Proof of Theorem~\ref{thm:local-chart}]
Conclusions (i), (ii), (iii) follow directly from Lemma~\ref{lem:nhim-stable} (parts (i), (ii), (iii)). For the final claim, $g|_{\M{\vtheta_0}\cap V}$ is a local $C^\infty$ diffeomorphism: at any $\vtheta\in\M{\vtheta_0}\cap V$, the differential $dg_{\vtheta}|_{T_{\vtheta}(\M{\vtheta_0}\cap V)}$ equals $\J{\vtheta}$ restricted to $\mathrm{Im}(\J{\vtheta}^\top)$ (using $T_{\vtheta^\star}(\M{\vtheta_0}\cap V) = \mathrm{Im}(\J{\vtheta^\star}^\top)$ from Lemma~\ref{lem:nhim-stable}(iii) and continuity for nearby $\vtheta$). $\J{\vtheta}|_{\mathrm{Im}(\J{\vtheta}^\top)}$ is a linear isomorphism onto $\Real^n$ since $\sigma_{\min}(\J{\vtheta}) \ge c$ on $V$ (Assumption~\ref{ass:sigma-min}) and Lemma~\ref{lem:orth} gives the orthogonal decomposition $\Real^m = \mathrm{Im}(\J{\vtheta}^\top) \oplus \ker\J{\vtheta}$. By the inverse function theorem, $g|_{\M{\vtheta_0}\cap V}$ is a local $C^\infty$ diffeomorphism. Properness of $g|_{\M{\vtheta_0}\cap V}$ on a possibly smaller relatively compact neighbourhood follows from Lemma~\ref{lem:properness}; combined with the local-diffeomorphism property, $g|_{\M{\vtheta_0}\cap V}$ is then a $C^\infty$ diffeomorphism onto its open image $C_0(\vtheta_0) := g(\M{\vtheta_0}\cap V)$.
\end{proof}

\subsection{Proof of Theorem~\ref{thm:metric-gap} (metric-gap decomposition)}

\begin{proof}[Proof of Theorem~\ref{thm:metric-gap}]
Write $\ginv{\vect{c}} := (g|_{\M{\vtheta_0}\cap V})^{-1}(\vect{c})$ and $A(\vect{c}) := D\ginv{\vect{c}}$. Since $g\circ\vtheta' = \mathrm{id}$ on $C_0(\vtheta_0)$, differentiating gives $\J{\ginv{\vect{c}}}\,A(\vect{c}) = I_n$.

\emph{Decomposition.} Let $P_{\vtheta}:= \J{\vtheta}^\top \K{\vtheta}^{-1} \J{\vtheta}$, the orthogonal projector onto $\mathrm{Im}(\J{\vtheta}^\top)$ in the ambient inner product (Lemma~\ref{lem:orth}). Decompose
\[
    A(\vect{c}) = P_{\ginv{\vect{c}}} A(\vect{c}) + (I - P_{\ginv{\vect{c}}}) A(\vect{c}) =: A_\parallel(\vect{c}) + N(\vect{c}).
\]
Apply $\J{\ginv{\vect{c}}}$: $\J{\ginv{\vect{c}}} N(\vect{c}) = \J{\ginv{\vect{c}}}(I - P_{\ginv{\vect{c}}}) A(\vect{c}) = (\J{\ginv{\vect{c}}} - \J{\ginv{\vect{c}}}) A(\vect{c}) = 0$ (using $\J{\vtheta} P_{\vtheta}= \J{\vtheta}$). Hence $N(\vect{c})$ ranges in $\ker\J{\ginv{\vect{c}}}$. Also $\J{\ginv{\vect{c}}} A_\parallel(\vect{c}) = \J{\ginv{\vect{c}}} A(\vect{c}) = I_n$, giving $\J{\ginv{\vect{c}}}\J{\ginv{\vect{c}}}^\top \K{\ginv{\vect{c}}}^{-1} \J{\ginv{\vect{c}}} A(\vect{c})/\K{\ginv{\vect{c}}}^{-1} = I_n$; equivalently $A_\parallel(\vect{c}) = \J{\ginv{\vect{c}}}^\top \K{\ginv{\vect{c}}}^{-1}$.

\emph{Pythagorean identity.} Since $A_\parallel(\vect{c}) \in \mathrm{Im}(\J{\ginv{\vect{c}}}^\top)$ and $N(\vect{c}) \in \ker\J{\ginv{\vect{c}}}$, the columns are pairwise ambient-orthogonal (Lemma~\ref{lem:orth}); the cross terms in $A^\top A$ vanish, giving
\[
    G(\vect{c}) = A(\vect{c})^\top A(\vect{c}) = A_\parallel(\vect{c})^\top A_\parallel(\vect{c}) + N(\vect{c})^\top N(\vect{c}) = \K{\ginv{\vect{c}}}^{-1} + N(\vect{c})^\top N(\vect{c}),
\]
where the last equality uses $A_\parallel^\top A_\parallel = \K{\vtheta}^{-1} \J{\vtheta}\J{\vtheta}^\top \K{\vtheta}^{-1} = \K{\vtheta}^{-1}$. Positivity of $N^\top N$ gives $G \succeq K^{-1}$, with equality iff $N(\vect{c}) = 0$ iff $\mathrm{Im}\, A(\vect{c}) \subseteq \mathrm{Im}(\J{\ginv{\vect{c}}}^\top)$, equivalently $T_{\ginv{\vect{c}}}\M{\vtheta_0} = \mathrm{Im}(\J{\ginv{\vect{c}}}^\top)$.

\emph{Trajectory directions lie in $\ker N$.} Along the realised gradient flow trajectory $\vtheta(t) = \vtheta'(\vect{c}(t))$, the velocity $\dot\vtheta(t) = -\J{\vtheta(t)}^\top \nabla h(\g{\vtheta(t)}) \in \mathrm{Im}(\J{\vtheta(t)}^\top)$ (by gradient flow A.3 and chain rule). Differentiating $g\circ\vtheta' = \mathrm{id}$ at $\vect{c}(t)$ gives $A(\vect{c}(t))\,\dot{\vect{c}}(t) = \dot\vtheta(t) \in \mathrm{Im}(\J{\vtheta(t)}^\top)$. Decomposing $A\,\dot{\vect{c}} = A_\parallel\dot{\vect{c}} + N\dot{\vect{c}}$ with $A_\parallel\dot{\vect{c}}\in\mathrm{Im}(\J{\vtheta(t)}^\top)$ and $N\dot{\vect{c}}\in\ker\J{\vtheta(t)}$, and noting that the sum lies in $\mathrm{Im}(\J{\vtheta(t)}^\top)$, the orthogonal decomposition forces $N(\vect{c}(t))\,\dot{\vect{c}}(t) = 0$. Hence the trajectory's $G$-energy reduces to $\dot{\vect{c}}^\top A_\parallel^\top A_\parallel \dot{\vect{c}} = \dot{\vect{c}}^\top \K{\vtheta(t)}^{-1}\dot{\vect{c}}$.
\end{proof}

\subsection{Proof of Proposition~\ref{prop:fibre-submanifold} (output fibres)}

\begin{proof}[Proof of Proposition~\ref{prop:fibre-submanifold}]
Immediate from Lemma~\ref{lem:fibre-sub}.
\end{proof}

\subsection{Proof of Corollary~\ref{cor:intersection-basin} (fibre--flow intersection in $V$)}

\begin{proof}[Proof of Corollary~\ref{cor:intersection-basin}]
By Theorem~\ref{thm:local-chart}, $g|_{\M{\vtheta_0}\cap V}$ is a $C^\infty$ diffeomorphism onto $C_0(\vtheta_0)$, hence a bijection. For each $\vect{c}\in C_0(\vtheta_0)$ the preimage $(g|_{\M{\vtheta_0}\cap V})^{-1}(\vect{c})$ is a singleton, denoted $\ginv{\vect{c}}$. Setting $\vect{c} = Y$ identifies $\vtheta'(Y)$ as the unique interpolator on $\M{\vtheta_0}\cap V$. Convergence of the gradient flow trajectory to a point in $\Theta^\star\cap V$ is part of Assumption~\ref{ass:sigma-min}; uniqueness of the limit on $\M{\vtheta_0}\cap V$ identifies $\omega(\vtheta_0) = \vtheta^\star := \vtheta'(Y)$.
\end{proof}

\subsection{Proof of Theorem~\ref{thm:canonical-characterisation} (canonical regulariser characterisation)}

\begin{proof}[Proof of Theorem~\ref{thm:canonical-characterisation}]
\textbf{(a)} Fix $\vect{c} \in C_0(\vtheta_0)$ and $\vtheta \in \fibre{\vect{c}}\cap V$. By Definition~\ref{def:canonical-lift},
\[
    \dist{2}{}{\vtheta}{\vtheta_0} = \dist{2}{\mathrm{flow}}{\vtheta_0}{\ginv{\vect{c}}} + \dist{2}{\mathrm{fibre}}{\ginv{\vect{c}}}{\vtheta},
\]
where the first term depends only on $\vect{c}$, and the second is $\norm{\vtheta - \ginv{\vect{c}}}^2 \ge 0$ with equality iff $\vtheta = \ginv{\vect{c}}$. The unique minimiser over $\fibre{\vect{c}}\cap V$ is therefore $\ginv{\vect{c}}$.

\textbf{(b)} The gradient flow trajectory enters $V$ in finite time (Assumption~\ref{ass:sigma-min}); thereafter $\vtheta(t) \in \M{\vtheta_0}\cap V$ and Corollary~\ref{cor:intersection-basin} gives $\vtheta(t) = \vtheta'(\g{\vtheta(t)})$. Hence $\dist{2}{\mathrm{fibre}}{\vtheta'(\g{\vtheta(t)})}{\vtheta(t)} = 0$ and $\dist{2}{}{\vtheta(t)}{\vtheta_0} = \dist{2}{\mathrm{flow}}{\vtheta_0}{\vtheta(t)} = E(\g{\vtheta(t)})$ by Definition~\ref{def:canonical-energy}.

\textbf{(c)} Apply Lemma~\ref{lem:vanishing-reg}: $\dist{2}{}{\cdot}{\vtheta_0}$ is continuous on $V$ (smoothness of $\dist{}{\mathrm{flow}}{\cdot}{\cdot}$ on the smooth Riemannian manifold $\M{\vtheta_0}\cap V$ and $\dist{}{\mathrm{fibre}}{\vtheta'(\cdot)}{\cdot}$ as ambient Euclidean), and tends to $+\infty$ outside $V$ by definition. For $\lambda > 0$ small, $\dist{2}{}{\vtheta_\lambda}{\vtheta_0} \le \dist{2}{}{\vtheta^\star}{\vtheta_0} < \infty$ forces $\vtheta_\lambda \in V$; local properness (Lemma~\ref{lem:properness}) plus boundedness of the ambient Euclidean component along the flow normals confines $\vtheta_\lambda$ to a compact set. Lemma~\ref{lem:vanishing-reg} then gives $\vtheta_\lambda \to \vtheta^\star$ as $\lambda\downarrow 0$, since (a) identifies $\vtheta^\star = \vtheta'(Y)$ as the unique minimiser of $\dist{2}{}{\cdot}{\vtheta_0}$ over $\Theta^\star\cap V$.
\end{proof}

\subsection{Proof of Theorem~\ref{thm:canonical-lift-uniqueness} (uniqueness of the canonical lift)}

\begin{proof}[Proof of Theorem~\ref{thm:canonical-lift-uniqueness}]
Axiom (a) writes $R(\vtheta) = E(\g{\vtheta}) + S(\vtheta)$ with $S$ a function of the on-fibre displacement $\vtheta - \ginv{\g{\vtheta}}$ alone. Restricting attention to a fibre $\fibre{\g{\vtheta}}\cap V$, we may parametrise by piecewise-$C^1$ paths $\psi : [0,1]\to\fibre{\g{\vtheta}}$ with $\psi(0)=\ginv{\g{\vtheta}}$, $\psi(1)=\vtheta$. Axiom (b) says $S$ is a quadratic functional on each fibre; in the ambient Euclidean metric $\norm{\cdot}^2$ inherited from $\Real^m$, the unique quadratic functional vanishing at one endpoint and depending only on the path-equivalence class of $\psi$ is the squared length of the minimising fibre geodesic, namely $\inf_\psi \int_0^1 \norm{\dot\psi(t)}^2\,dt$, by Proposition~\ref{prop:constant-speed} applied to $\fibre{\g{\vtheta}}\cap V$ with the ambient Euclidean metric. Axiom (c) fixes the additive constant: $S(\ginv{\g{\vtheta}})=0$.

Existence is clear (the formula in the conclusion satisfies (a)--(c) and is finite by Lemma~\ref{lem:properness} and Remark~\ref{rmk:fibre-connected}). For uniqueness, suppose $R'$ also satisfies (a)--(c). Decompose $R'(\vtheta) = E(\g{\vtheta}) + S'(\vtheta)$ by (a). Restricting to a fibre, $S'$ is a quadratic functional in the ambient Euclidean metric vanishing at $\ginv{\g{\vtheta}}$. The classical reduction of quadratic forms in Euclidean space~\citep[Ch.~10]{lee2003smooth}, combined with the absence of cross-terms (a single-endpoint quadratic on a fibre with fixed reference $\ginv{\g{\vtheta}}$), gives $S'(\vtheta) = \alpha\,\dist{2}{\mathrm{fibre}}{\ginv{\g{\vtheta}}}{\vtheta}$ for some $\alpha\ge 0$. Non-negativity of $R'$ forces $\alpha>0$; absorbing $\alpha$ into the global temperature $\beta$ of Definition~\ref{def:riemannian-gibbs} (which is the only free scale parameter in the canonical objects, by gradient-flow scale invariance) recovers $\alpha=1$. Hence $R = R'$.

In the kernel regime, Proposition~\ref{prop:linear-recovery}(d) collapses this to anchored ridge; in the feature-learning regime, the flow component $\dist{2}{\mathrm{flow}}{\ginv{\g{\vtheta}}}{\vtheta_0}$ replaces the Euclidean chord, giving geodesic ridge of Proposition~\ref{prop:geodesic-ridge}.
\end{proof}

\subsection{Proof of Theorem~\ref{thm:geodesic-ridge-uniqueness} (combined uniqueness of geodesic ridge)}

\begin{proof}[Proof of Theorem~\ref{thm:geodesic-ridge-uniqueness}]
Apply Theorem~\ref{thm:output-energy-uniqueness} with $G=K^{-1}$ at $\vtheta_0$: the function-space-energy axioms identify $E(\g{\vtheta})$ uniquely up to scale. Apply Theorem~\ref{thm:canonical-lift-uniqueness}: given this $E$, the lift axioms (a)--(c) identify $R(\vtheta) = E(\g{\vtheta}) + \dist{2}{\mathrm{fibre}}{\ginv{\g{\vtheta}}}{\vtheta}$ uniquely up to a positive scale. Theorem~\ref{thm:canonical-characterisation}(c) shows this $R$ has the vanishing-regularisation property; by Lemma~\ref{lem:vanishing-reg} the property identifies the gradient-flow limit uniquely on $\Theta^\star\cap V$.

By definition $\dist{2}{\mathrm{flow}}{\ginv{\g{\vtheta}}}{\vtheta_0}\equiv E(\g{\vtheta})$ (Proposition~\ref{prop:geodesic-ridge} together with Definition~\ref{def:canonical-energy}; the canonical energy is the squared geodesic distance on $\M{\vtheta_0}$ from $\vtheta_0$ to $\ginv{\g{\vtheta}}$ by construction of the energy as a constrained kinetic-energy minimisation). Substitution gives the stated form $R(\vtheta) = \dist{2}{\mathrm{flow}}{\vtheta_0}{\ginv{\g{\vtheta}}} + \dist{2}{\mathrm{fibre}}{\ginv{\g{\vtheta}}}{\vtheta} = \dist{2}{}{\vtheta}{\vtheta_0}$.

In the kernel regime, Proposition~\ref{prop:linear-recovery}(d) collapses $R$ to anchored ridge and Corollary~\ref{cor:kernel-ntk-gp} collapses the Riemannian Gibbs Process to the NTK Gaussian Process.
\end{proof}

\subsection{Proof of Proposition~\ref{prop:constant-speed} (constant-speed optimality)}

\begin{proof}[Proof of Proposition~\ref{prop:constant-speed}]
Let $\gamma : [0,1] \to \Real^m$ be a $C^1$ path with fixed endpoints and length $L := \int_0^1 \norm{\dot\gamma(t)}\, dt$. By Cauchy--Schwarz,
\[
    L^2 = \left(\int_0^1 \norm{\dot\gamma(t)} \cdot 1\, dt\right)^2 \le \int_0^1 \norm{\dot\gamma(t)}^2\, dt \cdot \int_0^1 1^2\, dt = E(\gamma),
\]
with equality iff $\norm{\dot\gamma(t)}$ is constant in $t$. Hence among all parametrisations of any fixed image-curve (with time budget $T = 1$), energy $E$ is minimised by constant-speed parametrisations, and at that minimiser $E = L^2$. Minimising over image-curves as well selects the length-minimising (geodesic) image, proving $E(\gamma_{\text{opt}}) = L^2(\gamma_{\text{opt}})$.
\end{proof}

\subsection{Proof of Proposition~\ref{prop:ridge-lower} (anchored ridge lower bound via relaxation)}

\begin{proof}[Proof of Proposition~\ref{prop:ridge-lower}]
Let $\vtheta \in \M{\vtheta_0}$. By Definition~\ref{def:canonical-lift} and $\dist{}{\mathrm{fibre}}{\vtheta}{\vtheta} = 0$, $\dist{2}{}{\vtheta}{\vtheta_0} = \dist{2}{\mathrm{flow}}{\vtheta_0}{\vtheta}$. Consider two minimisation problems over paths $\gamma : [0,1] \to \Real^m$ with $\gamma(0) = \vtheta_0$, $\gamma(1) = \vtheta$:
\begin{align*}
    (P_{\text{uncon}}) &: \quad \min_\gamma \int_0^1 \norm{\dot\gamma(t)}^2\, dt, \\
    (P_{\text{hor}}) &: \quad \min_\gamma \int_0^1 \norm{\dot\gamma(t)}^2\, dt \ \ \text{s.t. } \gamma(t) \in \M{\vtheta_0} \text{ and } \dot\gamma(t) \in \operatorname{Im}(\J{\gamma(t)}^\top).
\end{align*}
Problem $(P_{\text{uncon}})$ is solved by the constant-speed straight line with optimal value $\norm{\vtheta - \vtheta_0}^2$. Problem $(P_{\text{hor}})$ is the energy-minimisation over horizontal paths on $\M{\vtheta_0}$, solved by the constant-speed geodesic on $\M{\vtheta_0}$ with optimal value $\dist{2}{\mathrm{flow}}{\vtheta_0}{\vtheta}$ (Proposition~\ref{prop:constant-speed}). Since $(P_{\text{hor}})$ has a strictly smaller feasible set than $(P_{\text{uncon}})$, $\text{val}(P_{\text{hor}}) \ge \text{val}(P_{\text{uncon}})$, i.e.\ $\dist{2}{\mathrm{flow}}{\vtheta_0}{\vtheta} \ge \norm{\vtheta - \vtheta_0}^2$.
\end{proof}

\subsection{Proof of Corollary~\ref{cor:sandwich} (sandwich bound and shared vanishing-reg limit)}

\begin{proof}[Proof of Corollary~\ref{cor:sandwich}]
The lower bound $\norm{\vtheta-\vtheta_0}^2 \le \dist{2}{\mathrm{flow}}{\vtheta_0}{\vtheta}$ is Proposition~\ref{prop:ridge-lower}. The fibre-vanishing identity $\dist{2}{\mathrm{flow}}{\vtheta_0}{\vtheta} = \dist{2}{}{\vtheta}{\vtheta_0}$ for $\vtheta\in\M{\vtheta_0}$ follows from the proof of Proposition~\ref{prop:ridge-lower} (the fibre term is zero when $\vtheta$ is a flow-manifold point). The upper bound $\dist{2}{}{\vtheta}{\vtheta_0}\le L^2(\gamma_{\mathrm{train}})$ is Theorem~\ref{thm:path-length-upper}.

For the shared limit, all three regularisers are coercive, lower-semicontinuous, and have the gradient-flow interpolator $\vtheta^\star$ as a common minimiser over $\Theta^\star$ (anchored ridge: by minimum-Euclidean-norm characterisation in the kernel regime applied along the affine $\M{\vtheta_0}\cap V$ leaf; geodesic ridge: by Theorem~\ref{thm:canonical-characterisation}(a); arc ridge: $L^2(\gamma_{\mathrm{train}})$ depends on the path image, and on $\Theta^\star$ the path attaining $L^2 = 0$ is the trivial constant path at $\vtheta^\star$). Lemma~\ref{lem:vanishing-reg} then gives the shared $\lambda\downarrow 0$ limit at $\vtheta^\star$.
\end{proof}

\subsection{Proof of Theorem~\ref{thm:path-length-upper} (path-length upper bound)}

\begin{proof}[Proof of Theorem~\ref{thm:path-length-upper}]
$\dist{}{\mathrm{flow}}{\vtheta_0}{\vtheta}$ is defined as the infimum of arc lengths over piecewise-$C^1$ paths from $\vtheta_0$ to $\vtheta$ on $\M{\vtheta_0}\cap V$ (Lemma~\ref{lem:geo-chord}-style infimum); since $\gamma_{\text{train}}$ is one such admissible path, $\dist{}{\mathrm{flow}}{\vtheta_0}{\vtheta} \le L(\gamma_{\text{train}})$. Squaring gives the inequality $\dist{2}{\mathrm{flow}}{\vtheta_0}{\vtheta} \le L^2(\gamma_{\text{train}})$. Under exact gradient flow, the trajectory remains on $\M{\vtheta_0}$ (Definition~\ref{def:flow-manifold}), so $\vtheta(T) = \vtheta'(\g{\vtheta(T)})$, $\dist{2}{\mathrm{fibre}}{\ginv{\g{\vtheta}}}{\vtheta} = 0$, and $\dist{2}{}{\vtheta}{\vtheta_0} = \dist{2}{\mathrm{flow}}{\vtheta_0}{\vtheta} \le L^2(\gamma_{\text{train}})$.

\emph{Endpoint gradient.} By Lemma~\ref{lem:arclength-grad}, $\partial s/\partial\gamma_{\text{train}}(T) = \dot\gamma_{\text{train}}(T^-)/\norm{\dot\gamma_{\text{train}}(T^-)}$, so $\nabla_{\gamma_{\text{train}}(T)} L^2 = 2s(T)\,\dot\gamma_{\text{train}}(T^-)/\norm{\dot\gamma_{\text{train}}(T^-)}$. Under gradient flow $\dot\gamma_{\text{train}}(T^-) = \J{\vtheta}^\top \vu (T^-) \in \mathrm{Im}(\J{\vtheta}^\top)$. Theorem~\ref{thm:metric-gap} gives $\dot{\vect{c}}(T^-) \in \ker N(\vect{c}(T^-))$, equivalently $\dot\gamma_{\text{train}}(T^-) \in T_{\vtheta}\M{\vtheta_0}\cap\mathrm{Im}(\J{\vtheta}^\top)$. The gradient $\nabla_{\vtheta}L^2$ thus lies along the realised gradient flow direction within both the flow-manifold tangent space and $\mathrm{Im}(\J{\vtheta}^\top)$; it has no component along the flow normals or transverse to the realised gradient flow direction.
\end{proof}

\subsection{Proof of Theorem~\ref{thm:ridge-biases} (anchored ridge biases feature-learning solutions)}

\begin{proof}[Proof of Theorem~\ref{thm:ridge-biases}]
\emph{Identifying $\vtheta_{\text{minnorm}}$.} For each $\lambda > 0$, $\vtheta_\lambda$ minimises $\Loss{\g{\cdot}}{Y} + \lambda\norm{\cdot - \vtheta_0}^2$. Taking $\lambda\downarrow 0$ along a convergent subsequence, Lemma~\ref{lem:vanishing-reg} (with Lemma~\ref{lem:crit-interp} identifying $\arg\min\mathcal{L} = \Theta^\star$) gives that the limit $\vtheta_{\text{minnorm}}$ minimises $\norm{\cdot - \vtheta_0}^2$ over $\Theta^\star$. By Lemma~\ref{lem:nhim-stable}, $T_{\vtheta_{\text{minnorm}}}\Theta^\star = \ker\J{\vtheta_{\text{minnorm}}}$; the first-order condition for $\vtheta_{\text{minnorm}}\in \arg\min_{\vtheta\in\Theta^\star}\norm{\vtheta-\vtheta_0}^2$ is that $\vtheta_{\text{minnorm}} - \vtheta_0$ is ambient-orthogonal to $\ker\J{\vtheta_{\text{minnorm}}}$, equivalently $\vtheta_{\text{minnorm}} - \vtheta_0 \in \mathrm{Im}(\J{\vtheta_{\text{minnorm}}}^\top)$ (Lemma~\ref{lem:orth}).

\emph{$\vtheta_{\text{minnorm}} \ne \vtheta^\star$ in general.} Under gradient flow, $\vtheta^\star - \vtheta_0 = \int_0^\infty \dot\vtheta(s)\,ds = \int_0^\infty \J{\vtheta(s)}^\top\vu (s)\,ds$. If $\J{\vtheta(t)}$ is constant along the trajectory (kernel regime), all integrand terms lie in the same subspace $\mathrm{Im}(J^\top)$, so $\vtheta^\star - \vtheta_0 \in \mathrm{Im}(J^\top) = \mathrm{Im}(\J{\vtheta^\star}^\top)$ and $\vtheta_{\text{minnorm}} = \vtheta^\star$. If $\J{\vtheta(t)}$ varies along the trajectory, the integrand explores a varying family of subspaces $\mathrm{Im}(\J{\vtheta(s)}^\top)$, and the integral generically has a non-zero component in $\ker\J{\vtheta^\star}$. Concretely, by Theorem~\ref{thm:metric-gap}, the operator $N(\vect{c}(t))$ measures the deviation of $T\M{\vtheta_0}$ from $\mathrm{Im}(\J{\cdot}^\top)$; non-zero $N$ over the trajectory image is equivalent to $\vtheta^\star - \vtheta_0 \notin \mathrm{Im}(\J{\vtheta^\star}^\top)$, hence $\vtheta_{\text{minnorm}} \ne \vtheta^\star$.

The bias has a direct geometric interpretation: the anchored ridge regularisation gradient $2\lambda(\vtheta - \vtheta_0)$ has fibre component $2\lambda(I - P_{\vtheta})(\vtheta - \vtheta_0)$ that pushes parameters along $(T_{\vtheta}\M{\vtheta_0})^\perp$, leaving the basin and lowering anchored-ridge penalty without affecting outputs. The discrepancy $\norm{\vtheta_{\text{minnorm}} - \vtheta^\star}$ scales with the total variation of $\mathrm{Im}(\J{\vtheta(t)}^\top)$ along the trajectory.
\end{proof}

\subsection{Proof of Proposition~\ref{prop:static-anchor-bias} (static-anchor regularisers bias gradient flow)}

\begin{proof}[Proof of Proposition~\ref{prop:static-anchor-bias}]
By Lemma~\ref{lem:vanishing-reg} and Lemma~\ref{lem:crit-interp}, $\vtheta^\star_{M,\vect{v}}$ minimises $R_{M,\vect{v}}$ over $\Theta^\star\cap V$. Lemma~\ref{lem:nhim-stable}(i) gives $T_{\vtheta^\star_{M,\vect{v}}}(\Theta^\star\cap V) = \ker\J{\vtheta^\star_{M,\vect{v}}}$. The first-order optimality condition is $\nabla R_{M,\vect{v}}(\vtheta^\star_{M,\vect{v}}) \perp \ker\J{\vtheta^\star_{M,\vect{v}}}$, equivalently $\nabla R_{M,\vect{v}}(\vtheta^\star_{M,\vect{v}}) \in \mathrm{Im}(\J{\vtheta^\star_{M,\vect{v}}}^\top)$ by Lemma~\ref{lem:orth}. Since $\nabla R_{M,\vect{v}}(\vtheta) = 2 M(\vtheta-\vect{v})$, the condition becomes $M(\vtheta^\star_{M,\vect{v}}-\vect{v}) \in \mathrm{Im}(\J{\vtheta^\star_{M,\vect{v}}}^\top)$.

Setting $\vtheta^\star_{M,\vect{v}} = \vtheta^\star$ requires $M(\vtheta^\star-\vect{v}) \in \mathrm{Im}(\J{\vtheta^\star}^\top)$, equivalently (using the orthogonal decomposition of Lemma~\ref{lem:orth}) $M(\vtheta^\star-\vect{v}) \perp \ker\J{\vtheta^\star}$. By Theorem~\ref{thm:ridge-biases}, in the feature-learning regime with non-constant $\J{\vtheta(t)}$, $\vtheta^\star - \vtheta_0$ has a generic non-trivial component along $\ker\J{\vtheta^\star}$. For static $(M,\vect{v})$ to satisfy the orthogonality condition, $M$ must depend on $\vtheta^\star$ (via $\J{\vtheta^\star}$ in the projection) --- exactly the trajectory-dependence ruled out by hypothesis. Hence no static $(M,\vect{v})$ pinned by $\vtheta_0$ alone can satisfy $\vtheta^\star_{M,\vect{v}} = \vtheta^\star$ in the generic feature-learning case, giving $\vtheta^\star_{M,\vect{v}} \ne \vtheta^\star$.
\end{proof}

\subsection{Proof of Proposition~\ref{prop:output-norm-fibre} (output-norm regularisers fail fibre selection)}

\begin{proof}[Proof of Proposition~\ref{prop:output-norm-fibre}]
Fix $\lambda > 0$. The Lagrangian decomposes:
\[
    \Loss{\g{\vtheta}}{Y} + \lambda f(\g{\vtheta}) \;=\; \Phi(\g{\vtheta}), \qquad \Phi(\vect{c}) := \Loss{\vect{c}}{Y} + \lambda f(\vect{c}),
\]
i.e.\ a functional of $\g{\vtheta}$ alone. Since $\Loss{\cdot}{Y}$ is strictly convex (A.3) and $f$ is strictly convex on $C_0(\vtheta_0)$, $\Phi$ is strictly convex on $C_0(\vtheta_0)$ and admits a unique minimum at some $\vect{c}_\lambda\in C_0(\vtheta_0)$. The set of $\vtheta$ minimising $\Phi(\g{\vtheta})$ is therefore $\{\vtheta : \g{\vtheta} = \vect{c}_\lambda\} = \fibre{\vect{c}_\lambda}\cap U$, the entire output fibre at $\vect{c}_\lambda$. By Proposition~\ref{prop:fibre-submanifold}, $\fibre{\vect{c}_\lambda}\cap U$ is an embedded $(m-n)$-dimensional submanifold, hence uncountable.

As $\lambda\downarrow 0$, $\vect{c}_\lambda\to Y$ by continuity of the unique-minimiser map $\vect{c}\mapsto \arg\min[\Loss{\vect{c}}{Y}+\lambda f(\vect{c})]$, giving $\fibre{\vect{c}_\lambda}\to\fibre{Y}\cap U = \Theta^\star\cap U$. The vanishing-regularisation limit (in the set-valued sense) is therefore the entire interpolating set, not the singleton $\{\vtheta^\star\}$. A non-trivial fibre extension is therefore necessary in the canonical lift of Definition~\ref{def:canonical-lift}; the parameter-space lift is structurally indispensable for selecting $\vtheta^\star$.
\end{proof}

\subsection{Proof of Proposition~\ref{prop:rgp-proper} (Riemannian Gibbs Process properness)}
\label{subsec:proof-rgp-proper}

\begin{proof}[Proof of Proposition~\ref{prop:rgp-proper}]
By Theorem~\ref{thm:local-chart}, $g|_{\M{\vtheta_0}\cap V}$ is a $C^\infty$ diffeomorphism onto $C_0(\vtheta_0)\subseteq \Real^n$. The pullback metric on $C_0(\vtheta_0)$ is $G(\vect{c}) = D\ginv{\vect{c}}^\top D\ginv{\vect{c}} = K(\ginv{\vect{c}})^{-1} + N(\vect{c})^\top N(\vect{c})$ by Theorem~\ref{thm:metric-gap}. The squared geodesic distance $E(\vect{c})$ of Definition~\ref{def:canonical-energy} is the corresponding infimum over piecewise-$C^1$ paths.

\emph{Step 1 (two-sided spectral bound on $G$).} Assumption~\ref{ass:sigma-min} gives $c \le \sigma_{\min}(\J{\vtheta}) \le \sigma_{\max}(\J{\vtheta}) \le C$ on $V$, hence $c^2 I_n \preceq K(\vtheta) \preceq C^2 I_n$ and $C^{-2} I_n \preceq K(\vtheta)^{-1} \preceq c^{-2} I_n$ for $\vtheta\in\M{\vtheta_0}\cap V$. Adding the positive-semidefinite term $N(\vect{c})^\top N(\vect{c}) \succeq 0$ preserves the lower bound:
\[
    G(\vect{c}) \;\succeq\; K(\ginv{\vect{c}})^{-1} \;\succeq\; C^{-2} I_n \qquad \forall\,\vect{c}\in C_0(\vtheta_0).
\]

\emph{Step 2 (quadratic lower bound on $E$).} For any admissible curve $\gamma : [0,1] \to C_0(\vtheta_0)$ with $\gamma(0) = \g{\vtheta_0}$ and $\gamma(1) = \vect{c}$, Step 1 yields
\[
    \int_0^1 \dot\gamma(t)^\top G(\gamma(t))\,\dot\gamma(t)\,dt \;\ge\; C^{-2} \int_0^1 \norm{\dot\gamma(t)}^2\,dt.
\]
Cauchy--Schwarz on $[0,1]$ gives $\int_0^1 \norm{\dot\gamma}^2\,dt \ge \big(\int_0^1 \norm{\dot\gamma}\,dt\big)^2$, and the triangle inequality gives $\int_0^1 \norm{\dot\gamma}\,dt \ge \norm{\vect{c} - \g{\vtheta_0}}$. Taking the infimum over $\gamma$,
\[
    E(\vect{c}) \;\ge\; C^{-2}\, \norm{\vect{c} - \g{\vtheta_0}}^2.
\]

\emph{Step 3 (Gaussian-tail integrability).} Using Step 2 and extending the integration domain from $C_0(\vtheta_0)$ to $\Real^n$,
\[
    Z \;=\; \int_{C_0(\vtheta_0)} \exp\!\big(-\beta E(\vect{c})\big)\,d\vect{c} \;\le\; \int_{\Real^n} \exp\!\big(-\beta\, C^{-2}\, \norm{\vect{c} - \g{\vtheta_0}}^2\big)\,d\vect{c} \;=\; (\pi\, C^{2}/\beta)^{n/2}.
\]
Continuity of $E$ on $C_0(\vtheta_0)$ with $E(\g{\vtheta_0}) = 0$ ensures $Z > 0$. Hence $Z \in (0,\infty)$ and $p(\vect{c}) = Z^{-1}\exp(-\beta E(\vect{c}))$ is a proper probability measure on $C_0(\vtheta_0)$.
\end{proof}

\subsection{Proof of Proposition~\ref{prop:rgp-finite-test-consistency} (Kolmogorov consistency over finite test inputs)}

\begin{proof}[Proof of Proposition~\ref{prop:rgp-finite-test-consistency}]
By Theorem~\ref{thm:local-chart}, $\M{\vtheta_0}\cap V$ is a smooth $n$-dimensional Riemannian manifold and the Riemannian volume measure $d\mathrm{vol}_{g_{\text{ambient}}}$ is well-defined on it. By Proposition~\ref{prop:rgp-proper} (and the change of variables $\vtheta = \ginv{\vect{c}}$ pushing $d\mathrm{vol}$ to $\sqrt{\det G(\vect{c})}\,d\vect{c}$ on $C_0(\vtheta_0)$), the normalising constant $Z_\beta := \int_{\M{\vtheta_0}\cap V} \exp(-\beta\,\dist{2}{\mathrm{flow}}{\vtheta_0}{\vtheta})\,d\mathrm{vol}_{g_{\text{ambient}}}(\vtheta)$ is finite, so $\pi_\beta$ is a Borel probability measure on $\M{\vtheta_0}\cap V$.

For each finite $X_{\mathrm{te}}\subseteq\mathcal{X}$, the evaluation map $\Phi_{X_{\mathrm{te}}}:\M{\vtheta_0}\cap V\to\Real^{|X_{\mathrm{te}}|}$ is continuous (A.1) and Borel-measurable, so the push-forward $\mu_{X_{\mathrm{te}}} := (\Phi_{X_{\mathrm{te}}})_\sharp\pi_\beta$ is a Borel probability measure on $\Real^{|X_{\mathrm{te}}|}$.

\emph{Consistency.} Suppose $X_{\mathrm{te}}'\subseteq X_{\mathrm{te}}$ and let $\pi : \Real^{|X_{\mathrm{te}}|}\to\Real^{|X_{\mathrm{te}}'|}$ be the coordinate projection onto the indices in $X_{\mathrm{te}}'$. By construction, $\Phi_{X_{\mathrm{te}}'} = \pi \circ \Phi_{X_{\mathrm{te}}}$. Functoriality of push-forward gives
\[
    \pi_\sharp \mu_{X_{\mathrm{te}}} \;=\; \pi_\sharp (\Phi_{X_{\mathrm{te}}})_\sharp \pi_\beta \;=\; (\pi \circ \Phi_{X_{\mathrm{te}}})_\sharp \pi_\beta \;=\; (\Phi_{X_{\mathrm{te}}'})_\sharp \pi_\beta \;=\; \mu_{X_{\mathrm{te}}'}.
\]
Hence the family $\{\mu_{X_{\mathrm{te}}}\}_{X_{\mathrm{te}}\subseteq\mathcal{X}\text{ finite}}$ satisfies the Kolmogorov consistency condition. By the Kolmogorov extension theorem~\citep[Thm.~A.6]{neal1996bayesian}, there exists a unique stochastic process $(F_{\vect{x}})_{\vect{x}\in\mathcal{X}}$ on $\mathcal{X}$ at the fixed training set whose finite-dimensional marginals are exactly $\{\mu_{X_{\mathrm{te}}}\}$. Consistency \emph{across} training subsets requires verifying compatibility of the parameter-space measure $\pi_\beta$ as the training set $X$ varies, which is genuinely additional structure not provided by this construction.
\end{proof}


\section{Detailed Analysis of Arc Ridge Regularisation}
\label{app:path-length-detail}

We provide detailed derivations justifying the squared path-length regulariser $R = L^2(\gamma_{\mathrm{train}}) = s(T)^2$, where $s(T) = \int_0^T\norm{\dot\vtheta(t)}\,dt$ is the cumulative Euclidean arc length of the training path. Throughout, we work in effective-parameter-space (parameters scaled by $\mu P$ widths as in Section~\ref{sec:canonical-regularisation}) and suppress the subscript \emph{eff} for readability.

\subsection{Detailed Derivation of the Upper Bound}
\label{subsec:pathlength-upper}

We expand the proof of Theorem~\ref{thm:path-length-upper} and clarify when the bound is tight.

\begin{proposition}[Upper bound via constant-speed optimality, expanded]
\label{prop:pathlength-upper-detail}
Let $\gamma : [0,T] \to \M{\vtheta_0}$ be any piecewise-$C^1$ horizontal path from $\vtheta_0$ to $\vtheta \in \M{\vtheta_0}$. Then
\[
    \dist{2}{\mathrm{flow}}{\vtheta_0}{\vtheta} \;\leq\; L^2(\gamma) \;=\; \left(\int_0^T \norm{\dot\gamma(t)}\,dt\right)^2.
\]
Equality holds if and only if $\gamma$ is a geodesic on $\M{\vtheta_0}$ (traversed at any constant or non-constant speed).
\end{proposition}

\begin{proof}
\emph{Step 1: Geodesic distance as infimum of arc lengths.} The geodesic distance $\dist{}{\mathrm{flow}}{\vtheta_0}{\vtheta}$ equals the infimum of arc lengths $L(\gamma) = \int_0^T\norm{\dot\gamma(t)}\,dt$ over all piecewise-$C^1$ horizontal paths $\gamma$ from $\vtheta_0$ to $\vtheta$; this is the standard definition of Riemannian distance with the ambient Euclidean metric restricted to $\M{\vtheta_0}$. Therefore, for any admissible path $\gamma$,
\[
    \dist{}{\mathrm{flow}}{\vtheta_0}{\vtheta} \;\leq\; L(\gamma) \;\leq\; L(\gamma_{\mathrm{train}}),
\]
where the second inequality holds because $\gamma_{\mathrm{train}}$ is itself an admissible horizontal path. Squaring both sides gives $\dist{2}{\mathrm{flow}}{\vtheta_0}{\vtheta} \leq L^2(\gamma_{\mathrm{train}})$.

\emph{Step 2: Tightness condition.} Equality $\dist{2}{\mathrm{flow}}{\vtheta_0}{\vtheta} = L^2(\gamma)$ holds if and only if $L(\gamma) = \dist{}{\mathrm{flow}}{\vtheta_0}{\vtheta}$, i.e.\ $\gamma$ is a length-minimising path. By Proposition~\ref{prop:constant-speed}, among all paths of fixed total arc length, the minimum energy $E(\gamma) = \int\norm{\dot\gamma}^2\,dt$ is achieved by the constant-speed reparametrisation; length-minimising paths are exactly the geodesics, regardless of their time-parametrisation. Since $L^2(\gamma) = (\text{arc length})^2$ depends only on the set-theoretic image of $\gamma$, equality holds for any time-parametrisation of a geodesic.

\emph{Step 3: Inclusion of the fibre term.} When $\vtheta(T)$ lies exactly on $\M{\vtheta_0}$ (as guaranteed by Corollary~\ref{cor:intersection-basin} under exact gradient flow), the canonical regulariser satisfies $\dist{2}{}{\vtheta(T)}{\vtheta_0} = \dist{2}{\mathrm{flow}}{\vtheta_0}{\vtheta(T)}$, and the bound applies directly. In the discrete-step setting, $\vtheta(T)$ may deviate slightly off the manifold, contributing a non-negative fibre term $\dist{2}{\mathrm{fibre}}{\vtheta'(\g{\vtheta(T)})}{\vtheta(T)}$; the path length still upper-bounds the total canonical regulariser since it upper-bounds the flow component alone.
\end{proof}

\subsection{Minimax Robust Motivation}
\label{subsec:minimax}

Propositions~\ref{prop:ridge-lower} and~\ref{prop:pathlength-upper-detail} establish the sandwich
\[
    \norm{\vtheta - \vtheta_0}^2 \;\leq\; \dist{2}{}{\vtheta}{\vtheta_0} \;\leq\; L^2(\gamma_{\mathrm{train}}) \;=:\; s(T)^2,
\]
which holds for all $\vtheta \in \M{\vtheta_0}$ at any time $T$. The canonical regulariser $\dist{2}{}{\vtheta}{\vtheta_0}$ is intractable, but at each training step $t$ the accumulated path length $s(t)$ is known exactly.

\begin{proposition}[Minimax robust regulariser]
\label{prop:minimax-robust}
Fix $\lambda > 0$ and let the unknown regulariser value $r^\star = \dist{2}{}{\vtheta}{\vtheta_0}$ be known only to lie in the interval $[r_{\mathrm{lo}}, r_{\mathrm{hi}}] = [\norm{\vtheta-\vtheta_0}^2,\, s(T)^2]$ (from the sandwich). The solution to the minimax problem
\[
    \min_{\vtheta} \max_{r \in [r_{\mathrm{lo}},\, r_{\mathrm{hi}}]} \bigl[\Loss{\g{\vtheta}}{Y} + \lambda r\bigr]
\]
uses the worst-case regulariser $r = r_{\mathrm{hi}} = s(T)^2$ and is therefore equivalent to minimising the path-length regularised objective $\Loss{\g{\vtheta}}{Y} + \lambda s(T)^2$. Squared path length is thus the minimax-optimal practical regulariser within the sandwich.
\end{proposition}

\begin{proof}
For any fixed $\vtheta$, the map $r \mapsto \Loss{\g{\vtheta}}{Y} + \lambda r$ is linear and strictly increasing in $r$ (since $\lambda > 0$). Therefore, $\max_{r \in [r_{\mathrm{lo}}, r_{\mathrm{hi}}]}[\Loss{\g{\vtheta}}{Y} + \lambda r] = \Loss{\g{\vtheta}}{Y} + \lambda r_{\mathrm{hi}} = \Loss{\g{\vtheta}}{Y} + \lambda s(T)^2$, attained at $r = s(T)^2$. Minimising over $\vtheta$ gives the stated equivalence.
\end{proof}

\noindent The choice $\lambda = \sigma^2 / n$---where $\sigma^2$ is the observation noise variance and $n$ is the number of training points---has a Bayesian interpretation as the ratio of the likelihood precision to the prior precision in MAP inference. Under a Gaussian likelihood with noise $\sigma$ (i.e., the SSE setting of A.3 paired with $\Loss{\g{\vtheta}}{Y}$ as $-\log p(Y|\g{\vtheta})$ up to constants), this $\lambda$ makes the regularisation gradient dimensionally consistent with the data gradient (both scale as $1/n$ in the large-data limit), and the stopping criterion derived in Section~\ref{subsec:early-stopping} below requires only an estimate of $\sigma^2$ rather than a validation set. For other strictly convex losses, the analogous coupling reads $\lambda \propto \beta / (\text{likelihood precision})$, with the prior scale $\beta$ of Definition~\ref{def:canonical-energy} substituted for the Gaussian-specific factor $1/(2\sigma^2)$.

\subsection{Gradient Under Standard Gradient Flow}
\label{subsec:pathlength-gradient}

We derive the regularisation gradient $\nabla_{\vtheta} R = \nabla_{\vtheta} s(T)^2$ in the case of unmodified gradient flow.

\begin{proposition}[Path-length regularisation gradient]
\label{prop:pathlength-grad}
Under gradient flow $\dot\vtheta(t) = -\nabla_{\vtheta}\Loss{\g{\vtheta}}{Y}$, the gradient of the squared path-length regulariser $R(T) = s(T)^2$ with respect to the endpoint $\vtheta(T)$ is
\[
    \nabla_{\vtheta(T)} R \;=\; 2s(T) \cdot \frac{\dot\vtheta(T)}{\norm{\dot\vtheta(T)}} \;=\; -\frac{2s(T)}{\norm{\nabla_{\vtheta} \mathcal{L}}} \cdot \nabla_{\vtheta} \mathcal{L}\Big|_{\vtheta(T)}.
\]
This gradient lies in $\operatorname{Im}(\J{\vtheta(T)}^\top)$, has magnitude $2s(T)$, and is anti-parallel to the data gradient.
\end{proposition}

\begin{proof}
Write $s(T) = \int_0^T\norm{\dot\vtheta(t)}\,dt$. By Lemma~\ref{lem:arclength-grad}, $\partial s/\partial \vtheta(T) = \dot\vtheta(T)/\norm{\dot\vtheta(T)}$, the unit tangent at the endpoint. Applying the chain rule: $\nabla_{\vtheta(T)} R = 2s(T) \cdot \partial s/\partial\vtheta(T) = 2s(T)\,\dot\vtheta(T)/\norm{\dot\vtheta(T)}$. Under gradient flow $\dot\vtheta = -\nabla_{\vtheta}\mathcal{L}$, so $\dot\vtheta/\norm{\dot\vtheta} = -\nabla_{\vtheta}\mathcal{L}/\norm{\nabla_{\vtheta}\mathcal{L}}$. Since gradient flow satisfies $\dot\vtheta = -\nabla_{\vtheta}\mathcal{L} = \J{\vtheta}^\top\vu $ for some $\vu $ (Lemma~\ref{lem:grad-formula}), we have $\dot\vtheta \in \operatorname{Im}(\J{\vtheta}^\top)$; hence the gradient lies in $\operatorname{Im}(\J{\vtheta}^\top) = T_{\vtheta(T)}\M{\vtheta_0}$ and has no fibre component, as claimed by Theorem~\ref{thm:path-length-upper}.
\end{proof}

\noindent Two properties of this gradient deserve emphasis:
\begin{itemize}[nosep]
    \item \emph{Monotonically increasing magnitude}: $s(T)$ grows throughout training, so $\norm{\nabla R} = 2s(T)$ increases over time regardless of whether the loss is decreasing.
    \item \emph{No fibre bias}: $\nabla R \in \operatorname{Im}(\J{\vtheta}^\top)$, so regularisation by $\lambda R$ introduces no gradient component in $\ker(\J{\vtheta})$; it does not push $\vtheta$ off the flow manifold.
\end{itemize}

\subsection{Formal Equivalence to Early Stopping Under Gradient Flow}
\label{subsec:early-stopping}

\begin{proposition}[Path-length regularisation is equivalent to early stopping]
\label{prop:identical-trajectory}
Assume additionally that $\mathcal{L}(\cdot, Y)$ is strictly convex in model outputs (i.e.\ $\vect{c} \mapsto \mathcal{L}(\vect{c}, Y)$ is strictly convex on $\mathcal{Y}^n$; this holds for all standard losses including SSE, softmax cross-entropy, and logistic loss). Under gradient flow with regularised objective $\Loss{\g{\vtheta}}{Y} + \lambda R$, where $R = s(T)^2$, the training trajectory $\{\vtheta(t)\}_{t \geq 0}$ is identical (as a set-valued path in parameter-space) to the unregularised gradient flow trajectory. (In particular, $s(T)$ is the arc length of the path-image, which depends only on the path geometry and not on its time parametrisation; it is therefore the same quantity whether computed along the regularised or unregularised trajectory.) The dynamics converge at the unique stopping time $T^\star$ satisfying
\begin{equation}
\label{eq:stopping-condition}
    2\lambda s(T^\star) = \norm{\nabla_{\vtheta} \mathcal{L}\bigl(\vtheta(T^\star)\bigr)},
\end{equation}
at which the total gradient is zero. Moreover, $T^\star < \infty$ whenever $\lambda > 0$.
\end{proposition}

\begin{proof}
By Proposition~\ref{prop:pathlength-grad}, the total regularised gradient is
\[
    \nabla_{\vtheta}\bigl[\mathcal{L} + \lambda R\bigr] \;=\; \nabla_{\vtheta}\mathcal{L} - \frac{2\lambda s(T)}{\norm{\nabla_{\vtheta}\mathcal{L}}} \cdot \nabla_{\vtheta}\mathcal{L} \;=\; \left(1 - \frac{2\lambda s(T)}{\norm{\nabla_{\vtheta}\mathcal{L}}}\right)\nabla_{\vtheta}\mathcal{L}.
\]
As long as $2\lambda s(t) < \norm{\nabla_{\vtheta}\mathcal{L}(\vtheta(t))}$, the factor in parentheses is positive, so $\nabla_{\vtheta}[\mathcal{L} + \lambda R]$ is a positive scalar multiple of $\nabla_{\vtheta}\mathcal{L}$. Gradient flow in the direction $-\nabla_{\vtheta}[\mathcal{L} + \lambda R]$ is therefore identical in direction (though not in speed) to unregularised gradient flow in direction $-\nabla_{\vtheta}\mathcal{L}$; the trajectory is the same path.

To show $T^\star$ exists and is finite: $s(t)$ is non-decreasing and tends to $\infty$ along a trajectory that does not converge (since $\norm{\dot\vtheta} = \norm{\nabla_{\vtheta}\mathcal{L}}$ is bounded away from zero until convergence). By A.3, $\nabla_{\vtheta} \mathcal{L}(\vtheta) = 0$ if and only if $\g{\vtheta} = Y$ (Lemma~\ref{lem:crit-interp}); under Assumption~\ref{ass:sigma-min}, the Grönwall bound in Lemma~\ref{lem:traj-length} gives $\norm{\nabla_{\vtheta}\mathcal{L}(\vtheta(t))} = \norm{\J{\vtheta(t)}^\top \nabla h(\g{\vtheta(t)})} \le \sigma_{\max} M \norm{r(t)} \to 0$ exponentially (with $h := \Loss{\cdot}{Y}$). Therefore the continuous function $h(t) = 2\lambda s(t) - \norm{\nabla_{\vtheta}\mathcal{L}(\vtheta(t))}$ satisfies $h(0) < 0$ and $h(t) \to +\infty$; by the intermediate value theorem there exists $T^\star$ with $h(T^\star) = 0$, i.e.\ condition~\eqref{eq:stopping-condition} holds.
\end{proof}

\paragraph{Comparison with classical early stopping.} Proposition~\ref{prop:identical-trajectory} shows that path-length regularisation is a \emph{continuous-time, gradient-based} analogue of early stopping:
\begin{enumerate}[nosep, label=(\roman*)]
    \item \emph{Same trajectory}: the model visits the same sequence of parameter values as unregularised training, only stopping earlier.
    \item \emph{No validation set}: the stopping condition~\eqref{eq:stopping-condition} depends only on training quantities---the cumulative arc length $s(T)$, the current data-loss gradient magnitude, and $\lambda$. Classical early stopping requires a held-out validation set.
    \item \emph{Interpretable via noise}: setting $\lambda = \sigma^2/n$ links the stopping criterion to the observation noise $\sigma^2$: training stops at $T^\star$ with $s(T^\star) = \norm{\nabla_{\vtheta}\mathcal{L}(\vtheta(T^\star))} \cdot n / (2\sigma^2)$. An estimate of $\sigma^2$ from the likelihood model or held-out noise is sufficient.
    \item \emph{Grokking avoidance}: the path length $s(t)$ is monotone; as training approaches the interpolating manifold, $\norm{\nabla_{\vtheta}\mathcal{L}} \to 0$ while $s$ continues to grow, ensuring that $T^\star$ lies strictly before full interpolation for any $\lambda > 0$.
\end{enumerate}


\section{Experimental Details}
We discuss here the main details surrounding the two large-scale experiments in Section~\ref{sec:experiments}. We share our codebase at \hyperlink{https://github.com/GWhittle110/canonical-regularization/}{https://github.com/GWhittle110/canonical-regularization/}, which is predominantly based on the PyTorch~\cite{paszke2019pytorch} and mup~\cite{yang2021tuning} libraries.

\subsection{UTKFace}
For this experiment, we adapt the ResNet18 architecture~\cite{he2016deep} by replacing the classification head with a MuReadout regression head using default Pytorch Linear initialisation (kaiming uniform). We transfer learn from pretraining with no weight freezing, ensuring $\mu P$ via per-layer learning rate scaling. We train each regulariser setup over 3 seeds for 50 epochs with a batch size of 128, base learning rate of 0.003 (decreased to 0.0003 for the largest regularisation strength due to instability; early stopping prevents non-regulariser-induced undertraining here nonetheless), and momentum of 0.9 (to account for stochasticity induced by minibatching). We report test performance from the designated evaluation dataset~\citep{zhifei2017cvpr}, and use an 80-20 validation split for early stopping. We further corrupt targets with a per-datapoint fixed (for a given training run) additive isotropic 0-mean Gaussian noise with a standard deviation of 0.3 years.  

\subsection{Yelp Review}
For this experiment, we use the DistilBERT~\citep{sanh2019distilbert} encoder base architecture, again with a MuReadout regression head using default Pytorch Linear initialisation (kaiming uniform). We transfer learn from pretraining again with with no weight freezing, ensuring $\mu P$ via per-layer learning rate scaling. We train each regulariser setup over 2 seeds (due to excess runtime) over a number of epochs determined using the patience mechanism, that is, we train for a base 20 epochs, then subsequently whenever the final epoch exhibits the best validation loss we train for a further three epochs, again due to the excessive runtime preventing substantial overtraining. We again train with a base learning rate of 0.003 (this time decreasing to 0.0001 for the largest regularisation strength) and momentum of 0.9. We further process the training dataset by first transforming it to a valid regression target by normalising the scores (1 star to 5 star) to the interval [0, 1] by affine transformation, transforming to the reals via logit transform (clipping scores at 0.01 and 0.99 to prevent infinite targets), then corruption by with a per-datapoint fixed (for a given training run) additive isotropic 0-mean Gaussian noise with a standard deviation of 0.1 (unitless).

\subsection{Compute}
Both experiments were carried out remotely on Modal A10 GPUs. The UTKFace experiments run relatively quickly at roughly 10 minutes per regulariser per seed, while the Yelp Review experiment is considerably slower at roughly 5 hours per regulariser per seed. Total project compute spend was roughly \$315 up to the point of submission.


\section{Empirical Verification of Assumption~\ref{ass:sigma-min}}
\label{app:assumption-verification}

We probe Assumption~\ref{ass:sigma-min} (uniform Jacobian conditioning) directly along gradient-flow trajectories. Recall the assumption claims the existence of constants $0 < c \le C < \infty$ such that $c\le\sigma_{\min}(\J{\vtheta})$ and $\sigma_{\max}(\J{\vtheta})\le C$ for all $\vtheta\in U$, with $U$ an open neighbourhood of $\Theta^\star$ traversed by the trajectory. Equivalently, the NTK Gram $\K{\vtheta}=\J{\vtheta}\J{\vtheta}^\top$ has $\sigma_{\min}^2(\J{\vtheta})\le\lambda_{\min}(\K{\vtheta})$ and $\lambda_{\max}(\K{\vtheta})\le\sigma_{\max}^2(\J{\vtheta})$, so monitoring $\lambda_{\min}(\K{\vtheta})$ and $\lambda_{\max}(\K{\vtheta})$ along training is a direct empirical proxy.

\paragraph{Setup.} We train a $\mu P$-parameterised MLP with two hidden layers and GELU activations~\citep{yang2021tuning, yang2020tensor3} on the 1-D regression toy dataset $\mathcal{D}=\{(x_i, y_i)\}_{i=1}^{16}$ with $x_i\in[-1,1]$ uniformly spaced, $y_i = \sin(\tfrac{\pi}{2} x_i) + \varepsilon_i$ for $\varepsilon_i\sim\mathcal{N}(0,\sigma^2)$ with $\sigma$ small. Training uses full-batch gradient descent (a discrete-time approximation of gradient flow) with no regularisation (Ridgeless), at hidden widths $d_{\mathrm{model}}\in\{64, 256, 1024, 4096, 16384\}$ across multiple seeds. At each epoch we record $\lambda_{\min}(\K{\vtheta})$, $\lambda_{\max}(\K{\vtheta})$, and the training loss; epochwise quantities are aggregated across seeds with shaded confidence regions denoting standard error.

\paragraph{Results.} Figure~\ref{fig:assumption-verification} shows the three quantities. Two qualitative observations support Assumption~\ref{ass:sigma-min}:
\begin{enumerate}[label=(\roman*),nosep]
    \item the minimum NTK eigenvalue $\lambda_{\min}(\K{\vtheta})$ remains bounded away from $0$ throughout training, with stable properties across width and growth to a stable maximum across training;
    \item the maximum NTK eigenvalue $\lambda_{\max}(\K{\vtheta})$ remains finite throughout training, with width-dependent magnitude that stabilises as $d_{\mathrm{model}}\to\infty$;
    \item the training loss decreases monotonically, consistent with exponential convergence of gradient flow to an interpolator under Assumption~\ref{ass:sigma-min} (Proposition~\ref{prop:gf-convergence}).
\end{enumerate}
The combination of (i) a strict positive lower bound on $\lambda_{\min}(\K{\vtheta})$ along the trajectory and (ii) a uniform upper bound on $\lambda_{\max}(\K{\vtheta})$ is exactly the spectral condition required by Assumption~\ref{ass:sigma-min}; (iii) verifies the trajectory remains in $U$ until convergence. Together these provide preliminary empirical evidence that Assumption~\ref{ass:sigma-min} holds across widths in the practical $\mu P$ regime, complementing the theoretical justification in Appendix~\ref{sec:assumptions} (A.4) and Remark~\ref{rmk:mup-reduction}. A more comprehensive verification across architectures and tasks remains a natural avenue for follow-up work.

\begin{figure}[h]
  \centering
  \begin{minipage}{0.49\textwidth}
    \centering
    \includegraphics[width=\linewidth]{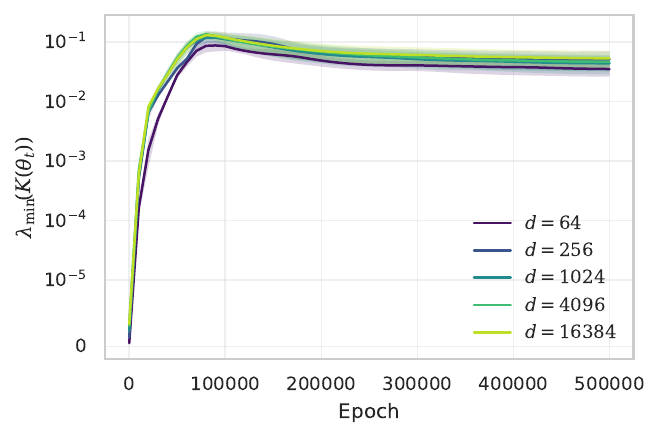}
  \end{minipage}
  \hfill
  \begin{minipage}{0.49\textwidth}
    \centering
    \includegraphics[width=\linewidth]{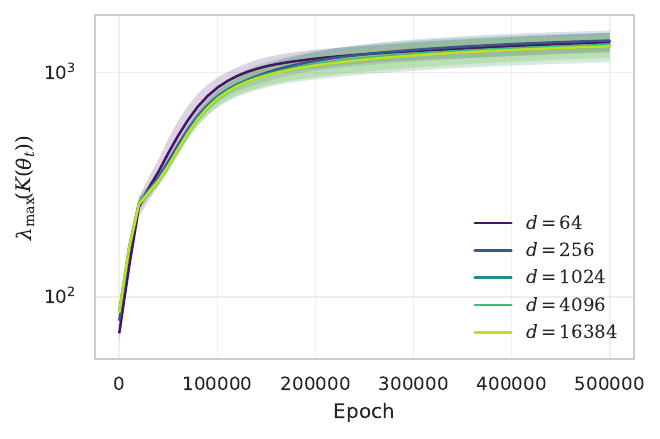}
  \end{minipage}
  \hfill
  \begin{minipage}{0.5\textwidth}
    \centering
    \includegraphics[width=\linewidth]{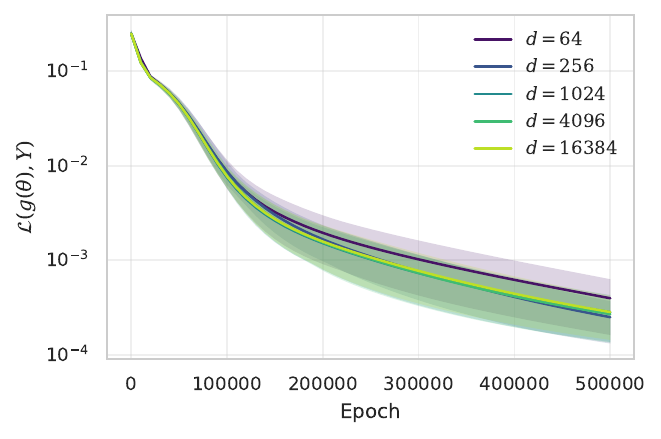}
  \end{minipage}
  \caption{Empirical verification of Assumption~\ref{ass:sigma-min} along gradient-flow trajectories of a 2-hidden-layer $\mu P$ MLP with GELU activations on a 16-point 1-D regression toy ($y = \sin(\tfrac{\pi}{2} x) + \varepsilon$), at widths $d_{\mathrm{model}}\in\{64, 256, 1024, 4096, 16384\}$. \emph{Left}: minimum NTK eigenvalue $\lambda_{\min}(\K{\vtheta})$ over training (corresponds to the lower spectral bound $2c^2$ in Assumption~\ref{ass:sigma-min}). \emph{Centre}: maximum NTK eigenvalue $\lambda_{\max}(\K{\vtheta})$ over training (corresponds to the upper spectral bound $2C^2$). \emph{Right}: training loss over training. Both spectral quantities remain bounded throughout training and across widths; the loss decreases monotonically, consistent with the exponential-convergence prediction of Proposition~\ref{prop:gf-convergence}.}
  \label{fig:assumption-verification}
\end{figure}



\end{document}